\documentclass[twoside,letterpaper,11pt]{article}
\usepackage{mystyle}
\ShortHeadings{Out-of-sample Extension for Latent Position
  Graphs}{Tang, Park, and Priebe}
\firstpageno{1}

\begin{document} 
\title{Out-of-sample Extension for Latent Position Graphs} 

\author{\name Minh Tang \email{mtang10@jhu.edu} \\ \name Youngser Park
  \email{youngser@jhu.edu} \\ \name Carey~E.~Priebe \email{cep@jhu.edu} \\ \addr Department of Applied Mathematics and
  Statistics \\ Johns Hopkins University \\ 3400 N. Charles St. \\ Baltimore, MD 21218,
  USA.}

\editor{}

\maketitle
\thispagestyle{empty}
\begin{abstract}%
We consider the problem of vertex classification for
graphs constructed from the latent position model. It was shown
previously that the approach of 
embedding the graphs into some Euclidean space followed
by classification in that space can yields a universally consistent
vertex classifier. However, a major technical difficulty
of the approach arises when classifying unlabeled out-of-sample
vertices without including them in the embedding stage. In this paper,
we studied the out-of-sample extension for the graph embedding step
and its impact on the subsequent inference tasks. We show that, under
the latent position graph model and for sufficiently large $n$, the
mapping of the out-of-sample vertices is close to its true latent
position. We then demonstrate that successful inference for the
out-of-sample vertices is possible.
\end{abstract}


\section{Introduction}
The classical statistical pattern recognition setting involves
\begin{equation*}
  (X,Y), (X_1, Y_1), (X_2, Y_2), \dots, (X_n, Y_n)
  \overset{\mathrm{i.i.d}}{\sim} F_{X, Y}
\end{equation*} 
where the $X_i \in \mathcal{X} \subset \mathbb{R}^{d}$
are observed feature vectors and the $Y_{i} \in \mathcal{Y} =
\{-1,1\}$ are observed class labels, for some probability distribution
$F_{X,Y}$ on $\mathcal{X} \times
\mathcal{Y}$. This setting has been extensively investigated and many
important and interesting theoretical concepts and results, e.g.,
universal consistency, structural complexities, and arbitrary slow
convergence are available. See, e.g., \cite{devroye1996probabilistic}
for a comprehensive overview.

Now, suppose that the feature vectors are unobserved, and we
observe instead a graph $G$ on $n+1$ vertices. Suppose also that $G$
is constructed in a manner such that there is a one-to-one
relationship between the vertices of $G$ and the feature vectors $X,
X_1, \dots, X_{n}$. The question of classifying the vertices based
on $G$ and the observed labels $Y_i$ then arises
naturally.

\begin{algorithm*}
  \begin{algorithmic} 
    \STATE \textbf{Input}: $\mathbf{A}\in
    \{0,1\}^{n\times n}$, training set $\mathcal{T} \subset [n] = 
    \{1,2,\dots,n\}$ and labels $\mathbf{Y}_{\mathcal{T}} = \{Y_i \colon i \in \mathcal{T}\}$. 
    \STATE \textbf{Output}: Class labels $\{ \hat{Y}_{j} \colon j \in [n] \setminus
    \mathcal{T}\}$. 
    \STATE \textit{Step 1}: Compute the
    eigen-decomposition of $\mathbf{A} =
    \mathbf{U}\mathbf{S}\mathbf{U}^T$. 
    \STATE \textit{Step 2}: Let $d$ be the ``elbow'' in
    the scree plot of $\mathbf{A}$, $\mathbf{S}_{\mathbf{A}}$ the
    diagonal matrix of the top $d$ eigenvalues of $\mathbf{A}$ and 
    $\mathbf{U}_{\mathbf{A}}$ the corresponding columns of
    $\mathbf{U}$.  
    \STATE \textit{Step 3}: Define
    $\mathbf{Z}$ to be 
    $\mathbf{U}_{\mathbf{A}}
    \mathbf{S}_{\mathbf{A}}^{1/2}$. Denote by $Z_i$ the $i$-th row of
    $\mathbf{Z}$. Define $\mathbf{Z}_{\mathcal{T}}$
    to be the rows of $\mathbf{Z}$ corresponding to the index set
    $\mathcal{T}$. $\mathbf{Z}$ is called the adjacency spectral
    embedding of $\mathbf{A}$. 
    \STATE
    \textit{Step 4}: Find a \emph{linear} classifier
    $\tilde{g}_n$ that minimizes the empirical
    $\varphi$-loss when trained on $(\mathbf{Z}_{\mathcal{T}},
    \mathbf{Y}_{\mathcal{T}})$ where $\varphi$ is a \emph{convex} loss
    function that is a surrogate for $0-1$ loss. 
    \STATE \textit{Step 5}: Apply $\tilde{g}_n$ on the $\{ Z_j \colon j \in [n]
    \setminus \mathcal{T}\}$ to obtain the $\{\hat{Y}_j \colon j \in
    [n] \setminus \mathcal{T}\}$. 
\end{algorithmic}
\caption{Vertex classifier on graphs}
\label{alg:main}
\end{algorithm*} 

A general approach to this classification problem is illustrated by
Algorithm~\ref{alg:main} wherein inference, e.g., classification or
clustering, proceeds by first embedding the graph $G$ into some
Euclidean space $\mathbb{R}^{d}$ followed by inference in that
space. This approach is well-represented in the literature of
multidimensional scaling, spectral clustering, and manifold
learning. The approach's popularity is due partly to its simplicity,
as after the embedding step, the vertices of $G$ are now points in
$\mathbb{R}^{d}$ and classification or clustering can proceed in an
almost identical manner to that of the classical setting, with a
plethora of well-established and robust inference procedures
available. In addition, theoretical justifications for the embedding
step are also available. For example in the spectral clustering and
manifold learning literature, the embedding step is often based on the
spectral decomposition of the (combinatorial or normalized) Laplacians
matrices of the graph. It can then be shown that, under mild
conditions on the construction of $G$, the Laplacian matrices converge
in some sense to the corresponding Laplace-Beltrami operators on the
domain. Thus, the eigenvalues and eigenvectors of the graph Laplacians
converge to the eigenvalues and eigenfunctions of the corresponding
operator. See for example
\cite{luxburg08:_consis,hein07:_conver_laplac,luxburg07,coifman06:_diffus_maps,belkin05:_towar_laplac,hein5:_from_laplac,singer06:_from_laplac,rosasco10:_integ_operat}
and the references therein for a survey of the results. 

The above cited results suggest that the embedding is conducive to the
subsequent inference task, but as they are general convergence results
and do not explicitly consider the subsequent inference problem, they
do not directly demonstrate that inference using the embeddings are
meaningful. Recently, there has been investigations that coupled the
embedding step with the subsequent inference step for several
widely-used random models for constructing $G$. For example,
\cite{rohe2011spectral,sussman12,fishkind2012consistent,chaudhuri12:_spect}
showed that clustering using the embeddings can be consistent for
graphs constructed based on the stochastic block model
\citep{Holland1983}, the random dot product model
\citep{young2007random}, and the extended partition model
\citep{karrer11:_stoch}. In related works,
\citet{sussman12:_univer,tangs.:_univer} showed that one can obtain
universally consistent vertex classification for graphs constructed based
on the random dot product model or its generalization, the latent
position model \citep{Hoff2002}.
However, a major technical difficulty of the approach arises when one
tries to use it to classify unlabeled \emph{out-of-sample}
vertices without including them in the embedding stage. A possible
solution is to recompute the embedding for each new vertex. However,
as many of the popular embedding methods are spectral in
nature, e.g., classical multidimensional scaling
\citep{torgersen52:_multid}, Isomap
\citep{tenebaum00:_global_geomet_framew_nonlin_dimen_reduc}, Laplacian
eigenmaps \citep{belkin03:_laplac} and diffusion maps
\citep{coifman06:_diffus_maps}, the computational costs for each new embedding
is of order $O(n^3)$, making this solution computationally
expensive. To circumvent this technical difficulty, 
out-of-sample extensions for many of the popular embedding methods
such as those listed above have been devised, see e.g.
\citet{faloutsos95,platt05:_fastm_metric_mds_nyst,bengio04:_out_lle_isomap_mds_eigen,williams01:_using_nystr,silva03:_global,wang99:_evaluat,trosset08}. In
these out-of-sample extensions, the embedding for the in-sample
points is kept fixed and the out-of-sample vertices are
inserted into the configuration of the in-sample points. The
computational costs are thus much less, e.g., linear in the number of
in-sample vertices for each insertion of an out-of-sample vertex into
the existing configuration. 

In this paper, we study the out-of-sample extension for the
embedding step in Algorithm~\ref{alg:main} and its impact on the
subsequent inference tasks. In particular we show
that, under the latent position graph model and for sufficiently
large $n$, the mapping of the out-of-sample vertices is close to its
true latent position. This suggests that inference for the
out-of-sample vertices is possible.

The structure of our paper is as follows. We introduce the framework
of latent position graphs in \S~\ref{sec:framework}. We describe the
out-of-sample extension for the adjacency spectral embedding and
analyze its properties in \S~\ref{sec:out-sample-estim}. In
\S~\ref{sec:experimental-results}, we investigate via simulation the
implications of performing inference using these out-of-sample
embeddings. We conclude the paper with discussion of related work, how
the results presented herein can be extended, and other implications.
\section{Framework}
\label{sec:framework} Let $\mathcal{X}$ be a compact metric space and
let $\kappa \colon \mathcal{X} \times \mathcal{X} \mapsto [0,1]$ be a
continuous positive definite kernel on $\mathcal{X}$. Let $F$ be a
probability measure on the Borel $\sigma$-field of $\mathcal{X}$. Now,
for a given $n$, let $X_1, X_2, \dots, X_n
\overset{\mathrm{i.i.d}}{\sim} F$. Let $\rho_n \in (0,1)$ be arbitrary
($\rho_n$ can depends on $n$). Define $\mathbf{K} = (\rho_n
\kappa(X_i,X_j))_{i,j=1}^{n}$.  Let $\mathbf{A}$ be a symmetric,
hollow, random binary matrix where the entries $\{\mathbf{A}_{ij}\}_{i
< j}$ of $\mathbf{A}$ are \emph{conditionally independent} Bernoulli
random variables given the $\{X_i\}_{i=1}^{n}$, with
$\Pr[\mathbf{A}_{ij} = 1] = \mathbf{K}_{ij} = \rho_n \kappa(X_i, X_j)$
for all $i,j \in \{1,2,\dots,n\}$, $i < j$. A graph $G$ whose
adjacency matrix $\mathbf{A}$ is constructed as above is an instance
of a latent position graph. The factor $\rho_n$ controls the sparsity
of the resulting graph. For example, if $\kappa > 0$ on $\mathcal{X}
\times \mathcal{X}$, then $\rho_n = (\log{n})/n$ leads to sparse,
connected graphs almost surely, $\rho_n = 1/n$ leads to graphs with a
single giant connected component, and $\rho_n = C > 0$ for some fixed
$C$ leads to dense graphs. We will denote by $G \sim
\mathrm{LPM}(\mathcal{X}, F, \kappa, \rho_n)$ an instance of a latent
position graph on $\mathcal{X}$ with distribution $F$, link function
$\kappa$, and sparsity factor $\rho_n$. We shall assume throughout
this paper that $n \rho_{n} = \omega(\log{n})$ for some $k \geq  
1$. That is, the expected average degree of $\mathbf{A}$ grows at least as fast
as $\omega(\log{n})$.     

An example of a latent position graph model is the random dot product
graph (RDPG) model of \cite{young2007random}. In the RDPG model,
$\mathcal{X}$ is taken to be the unit simplex in $\mathbb{R}^{d}$ 
and the link function $\kappa$ is the Euclidean inner
product. One can then take $F$ to be a Dirichlet distribution on the
unit simplex. Another example of a latent position graph model takes
$\mathcal{X}$ as a compact subset of $\mathbb{R}^{d}$ and the link
function $\kappa$ is a radial basis function, e.g., a Gaussian kernel
$\exp(-\|X_i - X_j\|^{2})$. This model is similar to the method of
constructing graphs based on point clouds in $\mathbb{R}^{d}$ in the
manifold learning literature. The small difference is that in the case
presented here, the Gaussian kernel is used for generating the edges
probabilities in the Bernoulli trials, i.e., the edges are unweighted
but random, whereas in the manifold learning literature, the Gaussian
kernel is used to assign weights to the edges i.e., the edges are
weighted but deterministic.

The latent position graph model and the related latent space
approach \citep{Hoff2002} is widely used in network
analysis. It is a generalization of the stochastic block model (SBM)
\citep{Holland1983} and variants such as the degree-corrected SBM
\citep{karrer11:_stoch}, the mixed-membership SBM \citep{Airoldi2008}
and the random dot product graph model \citep{young2007random}. It is
also closely related to the inhomogeneous random graph model
\citep{bollobas07} and the exchangeable graph model
\citep{diaconis08:_graph_limit_exchan_random_graph}. 

We now define a feature map $\Phi \colon \mathcal{X} \mapsto l_2$ for
$\kappa$. $\Phi$ will serve as our \emph{canonical} feature map, i.e.,
our subsequent results for the out-of-sample extension are based on
bounds for the deviation of the out-of-sample embedding from the
canonical feature map representation, e.g.,
Theorem~\ref{thm:2}. The kernel $\kappa$ defines an integral
operator $\mathcal{K}$ on $L^{2}(\mathcal{X}, F)$, the space of
$F$-square-integrable functions on $\mathcal{X}$, via
\begin{equation}
  \label{eq:7}
  (\mathcal{K}g)(x) = \int_{\mathcal{X}} \kappa(x,x') g(x') F(dx'). 
\end{equation}
$\mathcal{K}$ is then a compact operator and is of trace class (see
e.g., Theorem 4.1 in \cite{blanchard07:_statis}). Let
$\{\lambda_{j}(\mathcal{K})\}$ be the set of eigenvalues of
$\mathcal{K}$ in non-increasing order. The $\{\lambda_j\}$ are non-negative
and discrete, and their only accumulation point is at $0$. Let
$\{\psi_j\}$ be a set of orthonormal eigenfunctions of $\mathcal{K}$
corresponding to the $\{\lambda_j(\mathcal{K})\}$. Then by Mercer's
representation theorem \citep{cucker12}, one can write
\begin{equation*}
  \kappa(x,x') = \sum_{j=1}^{\infty} \lambda_{j} \psi_{j}(x) \psi_{j}(x')
\end{equation*}
with the above sum converging absolutely and uniformly for each $x$
and $x'$ in $\mathrm{supp}(F) \times \mathrm{supp}(F)$. We define the
feature map $\Phi \colon \mathcal{X} \mapsto l_2$ via
\begin{equation}
  \label{eq:13}
  \Phi(x) = (\sqrt{\lambda}_{j} \psi_{j}(x) \colon j = 1,2,\dots).
\end{equation}
We define a related feature map $\Phi_{d} \colon \mathcal{X} \mapsto
\mathbb{R}^{d}$ for $d \geq 1$ by
\begin{equation}
  \label{eq:15}
   \Phi_d(x) = (\sqrt{\lambda}_{j} \psi_{j}(x) \colon j = 1,2,\dots,d).
\end{equation}
We will refer to $\Phi_{d}$ as the \emph{truncated} feature map or as
the truncation of $\Phi$. We note that the feature map $\Phi$ and
$\Phi_d$ are defined in terms of the spectrum and eigenfunctions of
$\mathcal{K}$ and thus do not depend on the scaling parameter
$\rho_n$. 

We conclude this section with some notations that will be used in the
remainder of the paper. Let us denote by $\mathcal{M}_{d}$ and
$\mathcal{M}_{d,n}$ the set of $d \times d$ matrices and $d \times n$
matrices on $\mathbb{R}$, respectively. For a given adjacency matrix
$\mathbf{A} \in \mathcal{M}_{n}$, let $\mathbf{U}
\mathbf{S} \mathbf{U}^{T} $ be
the eigen-decomposition of $\mathbf{A}$. For a given $d \geq 1$, let
$\mathbf{S}_\mathbf{A}\in \mathcal{M}_{d}$ be the diagonal
matrix comprising of the $d$ largest eigenvalues of $\mathbf{A}$ and
let $\mathbf{U}_\mathbf{A}\in\mathcal{M}_{n,d}$ be the matrix
comprising of the corresponding eigenvectors. The matrices
$\mathbf{S}_\mathbf{K}$ are $\mathbf{U}_\mathbf{K}$ are defined
similarly. For a matrix $\mathbf{M}$, $\| \mathbf{M} \|$ denotes the
spectral norm of $\mathbf{M}$ and $\| \mathbf{M} \|_{F}$ denotes the
the Frobenius norm of $\mathbf{M}$. For a vector $\bm{v} \in
\mathbb{R}^{n}$, $v_i$ denote the $i$-th component of $v$ and
$\|\bm{v}\|$ denotes the Euclidean norm of $v$.
\section{Out-of-sample extension}
We now introduce the out-of-sample extension for the adjacency
spectral embedding of Algorithm~\ref{alg:main}.
\label{sec:out-sample-estim} 
\begin{definition}
  \label{def:1} Suppose $\mathbf{A}$ is an instance of
$\mathrm{LPM}(\mathcal{X}, F, \kappa, \rho_n)$ on $n$ vertices. Let
$\mathbf{Z} = \mathbf{U_{\mathbf{A}}} \mathbf{S}_{\mathbf{A}}^{1/2}
\in \mathcal{M}_{n,d}$ and denote by $\mathbf{Z}^{\dagger} \in
\mathcal{M}_{d,n}$ the matrix $(\mathbf{Z}^{T} \mathbf{Z})^{-1}
\mathbf{Z}^{T}$; $\mathbf{Z}^{\dagger}$ is the Moore-Penrose
pseudo-inverse of $\mathbf{Z}$. Let $Z^{\dagger}_{i}$ be the $i$-th
column of $\mathbf{Z}^{\dagger}$. For a given $X \in \mathcal{X}$, let
$T_{n}(X;\{X_i\}_{i=1}^{n})$ be the (random) mapping defined by
\begin{equation}
  \label{eq:2}
  T_n(X) := T_n(X;\{X_i\}_{i=1}^{n}) := \sum_{i=1}^{n} \xi_i
  Z^{\dagger}_{i} = \mathbf{Z}^{\dagger}\bm{\xi} 
\end{equation} where $\bm{\xi}$ is a vector of independent Bernoulli
random variables with $\Pr[\xi_i = 1] = \rho_n \kappa(X,X_i)$. The map $T_n$
is the out-of-sample extension of $X$; that is, $T_n(X)$ extends the
embedding $\hat{X}_i$ for the sampled $\{X_i\}_{i=1}^{n}$ to any $X
\in \mathcal{X}$.
\end{definition}
We make some quick remarks regarding Definition~\ref{def:1}. First, we
note that the out-of-sample extension give rise to i.i.d. random
variables, i.e., if $X'_1, X'_2, \dots, X'_m$ are
i.i.d from $F$, then the $T_n(X'_1;\{X_i\}_{i=1}^{n}),
T_n(X'_2;\{X_i\}_{i=1}^{n}), \dots, T_n(X'_m;\{X_i\}_{i=1}^{n})$ are
i.i.d. random variables in $\mathbb{R}^{d}$. Secondly, 
$T_n(X;\{X_i\}_{i=1}^{n})$ is a random mapping for any given $X$, even
when conditioned on the $\{X_i\}$. The randomness of $T_n$ arises 
from the randomness in the adjacency matrix $\mathbf{A}$ induced by
the in-sample points $\{X_i\}_{i=1}^{n}$ as well as the randomness in
the Bernoulli random variables $\bm{\xi}$ used in Eq.~\eqref{eq:2}. 
Thirdly, Eq.~\eqref{eq:2} states that
the out-of-sample extension $T_n(X)$ of $X$ is the least square
solution to $\| \mathbf{Z} \zeta - \bm{\xi} \|$, i.e., $\mathbf{Z}
T_n(X)$ is the least square projection of the (random) vector
$\bm{\xi}$ onto the subspace spanned by the columns of
$\mathbf{Z}$. The use of the least square solution to $\bm{\xi}$, or
equivalently the projection of $\bm{\xi}$ onto the subspace spanned by
the configuration of the in-sample points, is standard in many of the
out-of-sample extensions to the popular embedding methods, see e.g.
\cite{bengio04:_out_lle_isomap_mds_eigen,anderson03:_gener,faloutsos95,silva03:_global,wang99:_evaluat}. In
general, $\bm{\xi}$ is a vector containing the proximity (similarity
or dissimilarity) between the out-of-sample point and the in-sample
points and the least square solution can be related to the Nystr\"{o}m
method for approximating the eigenvalues and eigenvectors of a large
matrix, see e.g.
\cite{bengio04:_out_lle_isomap_mds_eigen,platt05:_fastm_metric_mds_nyst,williams01:_using_nystr}. 

Finally,
the motivation for Definition~\ref{def:1} can be gleaned by
considering the setting of random dot product graphs. In this setting,
$\mathbf{K} = \rho_{n} \mathbf{X} \mathbf{X}^{T}$ where $\mathbf{X}$ is the
matrix whose rows correspond to the sampled latent positions as points
in $\mathbb{R}^{d}$. Then $\tilde{\mathbf{Z}} =
\mathbf{U}_{\mathbf{K}} \mathbf{S}_{\mathbf{K}}^{1/2}$ is equivalent (up to
rotation) to $\rho_{n}^{1/2} \mathbf{X}$. Now let $\bm{\xi}$ be a
vector of Bernoulli random variables with $\mathbb{E}[\bm{\xi}] =
\mathbf{X} X$. Then
$\tilde{\mathbf{Z}}^{\dagger} \mathbb{E}[\bm{\xi}] = \rho_{n}^{-1/2}
\mathbf{X}^{\dagger} \rho_{n} \mathbf{X} X = \rho_{n}^{1/2}
X$. Thus, if we can show that $T_{n}(X;\{X_i\}_{i=1}^{n}) =
\mathbf{Z}^{\dagger} \bm{\xi} \approx \tilde{\mathbf{Z}}^{\dagger}
\mathbb{E}[\bm{\xi}]$, then we have $\rho_{n}^{-1/2} T_{n}(X;
\{X_i\}_{i=1}^{n}) \approx X$. As $\mathbf{Z}$ is ``close'' to
$\tilde{\mathbf{Z}}$ \citep{tangs.:_univer,sussman12:_univer} and
$(\mathbf{Z}^{\dagger} - \tilde{\mathbf{Z}}^{\dagger})(\bm{\xi} -
\mathbb{E}[\bm{\xi}])$ is ``small'' with high probability, see
e.g. \citet{tropp12:_user,yurinsky95:_sums_gauss}, 
the relationship $\rho_{n}^{-1/2} T_{n}(X;
\{X_i\}_{i=1}^{n}) \approx X$ holds for random dot product
graphs. As the latent position graphs with positive definite
kernels $\kappa$ can be thought of as being random dot product graphs
with latent positions being ``points'' in $l_2$, one expects a
relationship of the form $\rho_{n}^{-1/2} T_{n}(X;
\{X_i\}_{i=1}^{n}) \approx \Phi_{d}(X)$ for the (truncated) feature map $\Phi$ of
$\kappa$. Precise statements of the relationships are given in 
Theorem~\ref{thm:2} and Corollary~\ref{cor:1} below.  
\subsection{Out-of-sample extension and Nystr\"{o}m approximation} 
In the following discussion, we give a brief description of the
relationship between Definition~\ref{def:1} and the Nystr\"{o}m
approximation of \cite{drineas05:_nystr_gram,gittens:_revis_nystr}
which they called ``sketching''. Let $\mathbf{A} \in \mathcal{M}_{n}$
be symmetric and let $\mathbf{S} \in \mathcal{M}_{n,l}$ with $l \ll
n$. Following \cite{gittens:_revis_nystr}, let $\mathbf{C} =
\mathbf{A} \mathbf{S}$ and $\mathbf{A}_{\mathbf{S}} = \mathbf{S}^{T}
\mathbf{A} \mathbf{S}$. Then $\mathbf{C}
\mathbf{A}_{\mathbf{S}}^{\dagger} \mathbf{C}^{T}$ serves as a low-rank
approximation to $\mathbf{A}$ with rank at most $l$ and
\cite{gittens:_revis_nystr} refers to $\mathbf{S}$ as the sketching
matrix. The different choices for $\mathbf{S}$ leads to different
low-rank approximations. For example, a subsampling scheme correspond
to the entries of $\mathbf{S}$ being binaries $\{0,1\}$ variable with
a single non-zero entry in each row or column. More general entries
for $\mathbf{S}$ correspond to a linear projection of the columns of
$\mathbf{A}$. There are times when $\mathbf{A}_{\mathbf{S}}$ is
ill-conditioned and one is instead interested in the best rank $d$
approximation to $\mathbf{A}_{\mathbf{S}}$, i.e., the sketched version
of $\mathbf{A}$ is $\mathbf{C}
\tilde{\mathbf{A}}_{\mathbf{S}}^{\dagger} \mathbf{C}^{T}$ where
$\tilde{\mathbf{A}}_{\mathbf{S}}$ is a rank $d$ approximation to
$\mathbf{A}_{\mathbf{S}}$.

Suppose now that $\mathbf{S}$ correspond to a subsampling
scheme. Then $\mathbf{A}_{\mathbf{S}} = \mathbf{S}^{T} \mathbf{A} \mathbf{S}$
correspond to a sub-matrix of $\mathbf{A}$, i.e., 
$\mathbf{A}_{\mathbf{S}}$ correspond to the rows and columns indexed by
$\mathbf{S}$. Without loss of generality, we assume that $\mathbf{A}_{\mathbf{S}}$
is the first $l$ rows and columns of $\mathbf{A}$. That is, we have
the following decomposition
\begin{equation}
  \label{eq:40}
  \mathbf{A} = \begin{bmatrix} \mathbf{A}_{S,S} &
    \mathbf{A}_{S, S^{c}}
    \\ \mathbf{A}_{S^{c}, S} &
    \mathbf{A}_{S^{c},S^{c}} \end{bmatrix}
\end{equation}
where $S = \{1,2,\dots,l\}$ and $S^{c} = \{1,2,\dots,n\} \setminus
S$. We have abused notations slightly by writing
$\mathbf{A}_{\mathbf{S}} = \mathbf{A}_{S,S}$. Then $\mathbf{C}
\mathbf{A}_{\mathbf{S}}^{\dagger} \mathbf{C}^{T}$ can be written as
\begin{equation}
  \label{eq:41}
  \begin{split}
   \mathbf{C}
\mathbf{A}_{\mathbf{S}}^{\dagger} \mathbf{C}^{T} &= \begin{bmatrix} \mathbf{A}_{S,S} &
    \mathbf{A}_{S, S^{c}}
    \\ \mathbf{A}_{S^{c}, S} &
    \mathbf{A}_{S^{c},S^{c}} \end{bmatrix} \begin{bmatrix}
    \mathbf{I}_l \\ \bm{0} \end{bmatrix} \mathbf{A}_{S,S}^{\dagger}
  [ \mathbf{I}_l | \bm{0}] \begin{bmatrix} \mathbf{A}_{S,S} &
    \mathbf{A}_{S, S^{c}}
    \\ \mathbf{A}_{S^{c}, S} &
    \mathbf{A}_{S^{c},S^{c}} \end{bmatrix} \\
 &=\begin{bmatrix} \mathbf{A}_{S,S} \\ \mathbf{A}_{S^{c},
     S} \end{bmatrix} \mathbf{A}_{S,S}^{\dagger} [\mathbf{A}_{S,S} |
 \mathbf{A}_{S,S^{c}}] \\
 & = \begin{bmatrix} \mathbf{A}_{S,S} \mathbf{A}_{S,S}^{\dagger}
   \mathbf{A}_{S,S} & \mathbf{A}_{S,S} \mathbf{A}_{S,S}^{\dagger}
   \mathbf{A}_{S,S^{c}} \\ \mathbf{A}_{S^{c},S}
   \mathbf{A}_{S,S}^{\dagger} \mathbf{A}_{S,S} & \mathbf{A}_{S^{c},S^{c}}
   \mathbf{A}_{S,S}^{\dagger} \mathbf{A}_{S^{c},S^{c}} \end{bmatrix}.
  \end{split}
\end{equation} 
Let us now take $\tilde{\mathbf{A}}_{\mathbf{S}}$ to be
the best rank $d$ approximation to $\mathbf{A}_{\mathbf{S}}$ in the
positive semidefinite cone. Then $\mathbf{C}
\tilde{\mathbf{A}}_{\mathbf{S}}^{\dagger} \mathbf{C}^{T}$ can be
written as
\begin{equation}
  \label{eq:42}
  \mathbf{C} \tilde{\mathbf{A}}_{\mathbf{S}}^{\dagger}
\mathbf{C}^{T} = \begin{bmatrix} \mathbf{A}_{S,S} \tilde{\mathbf{A}}_{S,S}^{\dagger}
   \mathbf{A}_{S,S} & \mathbf{A}_{S,S} \tilde{\mathbf{A}}_{S,S}^{\dagger}
   \mathbf{A}_{S,S^{c}} \\ \mathbf{A}_{S^{c},S}
   \tilde{\mathbf{A}}_{S,S}^{\dagger} \mathbf{A}_{S,S} & \mathbf{A}_{S^{c},S^{c}}
   \tilde{\mathbf{A}}_{S,S}^{\dagger}
   \mathbf{A}_{S^{c},S^{c}} \end{bmatrix}. 
\end{equation} Now let $\mathbf{X} \in \mathcal{M}_{l,d}$ be such that
$\mathbf{X} \mathbf{X}^{T} = \mathbf{A}_{S,S}
\tilde{\mathbf{A}}_{S,S}^{\dagger} \mathbf{A}_{S,S} =
\tilde{\mathbf{A}}_{S,S}$ and let $\mathbf{Y} = (\mathbf{X}^{\dagger}
\mathbf{A}_{S,S^{c}})^{T} \in \mathcal{M}_{n-l,d}$. Then
Eq.~\eqref{eq:42}, can be written as
\begin{equation}
  \label{eq:43}
  \begin{split}
  \mathbf{C} \tilde{\mathbf{A}}_{\mathbf{S}}^{\dagger}
\mathbf{C}^{T} &= \begin{bmatrix} \mathbf{A}_{S,S} \tilde{\mathbf{A}}_{S,S}^{\dagger}
   \mathbf{A}_{S,S} & \mathbf{A}_{S,S} \tilde{\mathbf{A}}_{S,S}^{\dagger}
   \mathbf{A}_{S,S^{c}} \\ \mathbf{A}_{S^{c},S}
   \tilde{\mathbf{A}}_{S,S}^{\dagger} \mathbf{A}_{S,S} & \mathbf{A}_{S^{c},S^{c}}
   \tilde{\mathbf{A}}_{S,S}^{\dagger}
   \mathbf{A}_{S^{c},S^{c}} \end{bmatrix} \\
 &= \begin{bmatrix} \tilde{\mathbf{A}}_{S,S}   & \mathbf{A}_{S,S} \tilde{\mathbf{A}}_{S,S}^{\dagger}
   \mathbf{A}_{S,S^{c}} \\ \mathbf{A}_{S^{c},S}
   \tilde{\mathbf{A}}_{S,S}^{\dagger} \mathbf{A}_{S,S} & \mathbf{A}_{S^{c},S^{c}}
   \tilde{\mathbf{A}}_{S,S}^{\dagger}
   \mathbf{A}_{S^{c},S^{c}}  \end{bmatrix} \\
 &= \begin{bmatrix} \mathbf{X} \mathbf{X}^{T} & \mathbf{X}
   \mathbf{Y}^{T} \\ \mathbf{Y} \mathbf{X}^{T} & \mathbf{Y} \mathbf{Y}^{T}
   \end{bmatrix}.
 \end{split}
\end{equation}
We thus note that if $\mathbf{A}$ is an adjacency matrix on a graph
$G$ with $n$ vertices then $\mathbf{A}_{S}$ is the adjacency matrix of
the induced subgraph of $G$ on $l$ vertices. Then
$\tilde{\mathbf{A}}_{S} = \mathbf{U}_{\mathbf{A}}
\mathbf{S}_{\mathbf{A}} \mathbf{U}_{\mathbf{A}}$ is the rank $d$
approximation to $\mathbf{A}$ that arises from the adjacency spectral
embedding of $\mathbf{A}$. Thus $\mathbf{X} = \mathbf{U}_{\mathbf{A}}
\mathbf{S}_{\mathbf{A}}^{1/2}$ and therefore $\mathbf{Y} =
(\mathbf{X}^{\dagger} \mathbf{A}_{S,S^{c}})^{T}$ is the matrix each of whose rows
correspond to an out-of-sample embedding of the rows of
$\mathbf{A}_{S,S^{c}}$ into $\mathbb{R}^{d}$ as defined in
Definition~\ref{def:1}. 

In summary, in the context of adjacency spectral embedding, the
embeddings of the in-sample and out-of-sample vertices generate a
Nystr\"{o}m approximation to $\mathbf{A}$ and a Nystr\"{o}m
approximation to $\mathbf{A}$ can be used to derive the embeddings
(through an eigen-decomposition) for the in-sample and out-of-sample
vertices.   
\subsection{Estimation of feature map}
The main result of this paper is the following result on the
out-of-sample mapping error $T_n(X) - \Phi_d(X)$. Its proof is given
in the appendix. We note that the dependency on $\kappa$ and $F$ is
hidden in the spectral gap $\delta_{d}$ of $\mathcal{K}$, 
 the integral operator induced by $\kappa$ and $F$. 
\begin{theorem} Let $d \geq 1$ be given. Denote by $\delta_d$ the
quantity $\lambda_{d}(\mathcal{K}) - \lambda_{d+1}(\mathcal{K})$ and
suppose that $\delta_{d} > 0$. Let $\eta \in (0,1/2)$ be
arbitrary. Then there exists an orthogonal $\mathbf{W}$ such
that
  \begin{equation}
    \label{eq:3}
    \Pr\Bigl[\| \rho_n^{-1/2} \mathbf{W} T_n(X) - \Phi_{d}(X) \| \leq
    C \delta_{d}^{-3} \sqrt{d \frac{\log{(n/\eta)}}{n
      \rho_n }}\Bigr] \geq 1 - 2 \eta
  \end{equation} for some constant $C$ independent of $n$, $\eta$,
  $\kappa$, $d$, and $F$.
 \label{thm:2}
\end{theorem} 
We note the following corollary of the above result for
the case where the latent position model is the random dot product
graph model. For this case, the operator $\mathcal{K}$ is of rank $d$
and the truncated feature map $\Phi_d(X)$ is equal (up to rotation) to
the latent position $X$.
\begin{corollary}
  \label{cor:1}
  Let $\mathbf{A} \in \mathcal{M}_{n}$ be an instance of
  $\mathrm{RDPG}(\mathbb{R}^{d}, F)$. Denote by $\delta_{d}$ the
  smallest eigenvalue of $\mathbb{E}[XX^{T}]$. Let
  $\eta \in (0,1/2)$ be arbitrary. Then there exists an orthogonal
  $\mathbf{W}$ such that
   \begin{equation}
     \label{eq:16}
    \Pr\Bigl[\| \rho_n^{-1/2}\mathbf{W} T_n(X) - X \| \leq C
    \delta_{d}^{-3} \sqrt{d \frac{\log{(n/\eta)}}{n
      \rho_n }}\Bigr] \geq 1 - 2 \eta
  \end{equation} for some constant $C$ independent of $n$, $\eta$,
  $d$, and $F$.
  \label{cor:2}
\end{corollary}
We note the following result from \citet{tangs.:_univer} that
serves as an analogue of Theorem~\ref{thm:2} for the in-sample
points. We note that, other than the possibly different hidden constants
, the bound for the out-of-sample points in
Eq.~\eqref{eq:3} is almost identical to that of the in-sample points
in Eq.~\eqref{eq:49}. The main difference is in the power of the
spectral gap in the bounds, i.e., $\delta_{d}^{-3}$ against
$\delta_{d}^{-2}$. This difference might be due to the proof technique and 
not inherent in the distinction of out-of-sample versus in-sample
points. We also note that one can take the orthogonal matrix $\mathbf{W}$ for the
out-of-sample points to be the same as the in-sample points, i.e., the
rotation that makes the in-sample points ``close'' to the truncated
feature map $\Phi_d$ also makes the out-of-sample points ``close'' to
$\Phi_d$. 
\begin{theorem}
  \label{thm:3}
Let $d \geq 1$ be given. Denote by $\delta_{d}$ the
quantity $\lambda_{d}(\mathscr{K}) - \lambda_{d+1}(\mathscr{K})$ and
suppose that $\delta_{d} > 0$. Let $\eta \in (0,1/2)$ be
arbitrary. Let $\hat{\Phi}_{d}(X_i)$ denote the $i$-th row of
$\mathbf{U}_\mathbf{A}\mathbf{S}_\mathbf{A}^{1/2}$.
Then there exists a unitary matrix $\mathbf{W} \in
\mathcal{M}_{d}(\mathbb{R})$ such that for all $i \in [n]$
\begin{equation}
  \label{eq:49}
\mathbb{P}\Bigl[\|\rho_n^{-1/2} \mathbf{W} \hat{\Phi}_{d}(X_i)
-\Phi_{d}(X_i) \| \leq  C \delta_{d}^{-2} \sqrt{\frac{d \log{(n/\eta)}}{n
      \rho_n }}\Bigr] \geq 1 - 2 \eta
\end{equation}
for some constant $C$ independent of $n$, $\eta$, $\kappa$, $d$, and $F$. 
\end{theorem}
Theorem~\ref{thm:2} and its corollary states that in
the latent position model, the out-of-sample embedded points can be
rotated to be very close to the true feature map with high
probability. This suggest that successful statistical inference on the
out-of-sample points is possible. As an example, we investigate the
problem of vertex classification for latent position graphs whose link functions
$\kappa$ belong to the class of universal kernels. Specifically, we
consider an approach that proceeds by embedding the
vertices into some $\mathbb{R}^{d}$ followed by finding a linear
discriminant in that space. It was shown in \cite{tangs.:_univer} that
such an approach can be made to yield a universally consistent vertex
classifier if the vertex to be classified is embedded in-sample as 
the number of in-sample vertices increases to $\infty$. In the
following discussion we present a variation of this result in the case
where the vertex to be classified is embedded out-of-sample and the
number of in-sample vertices is fixed and finite. We 
show that under this out-of-sample setting, the misclassification
rate can be made arbitrarily small provided that the number of
in-sample vertices is sufficiently large (see Theorem~\ref{thm:1}).  

\begin{definition}
  \label{def:2}
  A continuous kernel on a metric space $\mathcal{X}$ is said to be a
  universal kernel, if for some choice of feature map $\Phi \colon
  \mathcal{X} \mapsto H$ of $\kappa$ to some Hilbert space $H$, the
  class of functions of the form 
  \begin{equation}
    \label{eq:17}
    \mathscr{F}_{\Phi} = \{ \langle w, \Phi \rangle_{H} \colon w \in H \}
  \end{equation}
  is dense in $\mathscr{C}(\mathcal{X})$, i.e., for any continuous $g
  \colon \mathcal{X} \mapsto \mathbb{R}$ and any $\epsilon > 0$, there
  exists $f \in \mathscr{F}_{\Phi}$ such that $\|f - g\|_{\infty} <
  \epsilon$.
\end{definition} We note that if $\kappa$ is such that
$\mathscr{F}_{\Phi}$ is dense in $\mathcal{C}(\mathcal{X})$ for some
feature map $\Phi$ of $\kappa$, then $\mathscr{F}_{\Phi'}$ is dense in
$\mathcal{C}(\mathcal{X})$ for any feature map $\Phi'$ of $\kappa$,
i.e., the universality of $\kappa$ is independent of its choice of
feature map. In addition, any feature map $\Phi$ of a universal kernel
$\kappa$ is injective. The following result lists several well-known
universal kernels. For more on universal kernels, the reader is
referred to
\cite{steinwart01:_suppor_vector_machin,micchelli06:_univer}.
\begin{proposition}
  \label{prop:2}
  Let $S$ be a compact subset of $\mathbb{R}^{d}$. Then the following
  kernels are universal on $S$.
  \begin{itemize}
  \item exponential kernel $\kappa(x,y) = \exp(\langle x, y \rangle)$.
  \item Gaussian kernel $\kappa(x,y) = \exp( - \| x - y 
    \|^{2}/\sigma^{2})$ for $\sigma > 0$.
  \item The binomial kernel $\kappa(x,y) = (1 - \langle x, y
    \rangle)^{-\alpha}$ for $\alpha > 0$.
  \item inverse multiquadrics $\kappa(x,y) = (c^{2} + \| x
    - y \|^{2})^{-\beta}$ for $\beta > 0$.
  \end{itemize}
\end{proposition}
Let $\mathcal{C}^{(d)}_{\Phi}$ be the class of linear functions on
$\mathbb{R}^{d}$ induced by the feature map $\Phi_{d}$ whose linear coefficients are normalized to have norm
at most $d$, i.e., $g \in \mathcal{C}^{(d)}_{\Phi}$ if and only if $g$ is of
the form
\begin{equation*}
  g(X) = \langle w, \Phi_{d}(X) \rangle_{\mathbb{R}^{d}}
\end{equation*}
for some $w \in \mathbb{R}^{d}, \|w\| \leq d$. We note that the
$\{\mathcal{C}_{\Phi}^{(d)}\}$ is a nested increasing sequence and
furthermore that
\begin{equation*}
 \bigcup_{d \geq 1} \mathcal{C}_{\Phi}^{(d)} = \mathcal{F}_{\Phi} =
\{ \langle w, \Phi \rangle_{\mathcal{H}} \colon w \in \mathcal{H}\}. 
\end{equation*}
Now, given $\{X_i\}_{i=1}^{n}$, let $\mathcal{C}^{(d)}_{T_n}$ be the
class of linear functions on $\mathbb{R}^{d}$ induced by the
out-of-sample extension $T_n(X;\{X_i\}_{i=1}^{n})$, i.e., $g \in
\mathcal{C}^{(d)}_{T_n}$ if and only if $g$ is of the form
\begin{equation}
  \label{eq:44}
  g(X) = \langle w, T_n(X;\{X_i\}_{i=1}^{n}) \rangle_{\mathbb{R}^{d}}.
\end{equation}
\begin{theorem}
  \label{thm:1}
  Let $\kappa$ be a universal kernel on $\mathcal{X}$. Let $\eta > 0$
  be arbitrary. Then for any
  $F_{X,Y}$ and any $\epsilon > 0$, there exists 
  $d$ and $n_0$ such that if $n \geq n_0$ 
  then
  \begin{equation}
    \label{eq:45}
    L^{*}_{T_n} \leq L^{*} + \epsilon,
  \end{equation} where $T_n \colon \mathcal{X} \mapsto \mathbb{R}^{d}$
is the out-of-sample mapping as defined in
Definition~\ref{def:1} and $L^{*}$ is the Bayes risk for the
classification problem with distribution $F_{X,Y}$.  
\end{theorem}
We make a brief remark regarding Theorem~\ref{thm:1}. The term
$L^{*}_{T_n}$ in Eq.~\eqref{eq:45} refers to the Bayes risk for the
classification problem given by the out-of-sample mapping $T_n$. As
noted earlier,  
$T_n(X;\{X_i\}_{i=1}^{n})$ is a random mapping for any given $X$, even when conditioned
on the $\{X_i\}$ as $T_n(X)$ also depends on the latent position graph
$\mathbf{A}$ generated by the $\{X_i\}$. As such, with slight abuse of
notations, $L^{*}_{T_n}$ refers
to the Bayes-risk of the mapping $T_n$ when not conditioned on any set
of $\{X_i\}$. That is, $L^{*}_{T_n}$ is the Bayes-risk for
out-of-sample embedding in the presence of $n$ in-sample latent
positions, i.e., the latent positions of the in-sample points are
integrated out. As the information processing
lemma implies $L^{*}_{T_n} \geq L^{*}$, one can view Eq.~\eqref{eq:45}
as a sort of converse to the information processing lemma in that
the degradation due to the out-of-sample embedding transformation $T_n$ can be made
negligible if the number of in-sample points is sufficiently large.
\begin{proof} 
Let $\varphi$ be any classification-calibrated (see
\cite{bartlett06:_convex}) convex surrogate of the $0-1$ loss. For any
measurable function $f \colon \mathcal{X} \mapsto \mathbb{R}$, let
$R_{\varphi,f}$ be defined by $\mathbb{E}[\varphi(Yf(X))]$. 
Let $f^{*}$ be a measurable function such that 
$R_{\varphi,f^{*}} = R_{\varphi}^{*} = \inf R_{\varphi,f}$
where the infimum is taken over the set of all measurable functions on
$\mathcal{X}$. As $\mathcal{F}_{\Phi}$ is dense in the set of
measurable functions on $\mathcal{X}$, without loss of generality we
can take $f^{*} \in \mathcal{F}_{\Phi}$. Now let $\epsilon > 0$ be
arbitrary. As $\mathcal{F}_{\Phi} = \bigcup_{d \geq 1}
\mathcal{C}^{(d)}_{\Phi}$, and the $\{\mathcal{C}^{(d)}_{\Phi}\}$ is a
nested increasing sequence, there exists a $d \geq 1$ such that for
some $\tilde{f} \in \mathcal{C}^{(d)}_{\Phi}$, we have $\|\tilde{f} -
f^{*} \|_{\infty} < \epsilon$. Thus, for any $\epsilon > 0$, there
exists a $d \geq 1$ such that for some $\tilde{f} \in
\mathcal{C}^{(d)}_{\Phi}, R_{\varphi,\tilde{f}} - R_{\varphi}^{*} <
\epsilon$. Now let $f \in \mathcal{C}^{(d)}_{\Phi}$ be arbitrary. Then $f =
\langle w, \Phi_{d} \rangle_{\mathbb{R}^{d}}$ for some $w \in
\mathbb{R}^{d}, \|w\| \leq d$. Let $g = \langle w, T_n
\rangle_{\mathbb{R}^d}$ and consider the difference $R_{\varphi,f} -
R_{\varphi,g}$. As $\varphi$ is convex, it is locally-Lipschitz and \begin{equation}
  \label{eq:46}
  \begin{split}
  |R_{\varphi,f} - R_{\varphi,g}| &= |\mathbb{E}[\varphi(Yf(X)) -
  \varphi(Yg(X))]| 
  \leq \mathbb{E}[|M(Yf(X) - Yg(X))|]  \\
  & \leq \mathbb{E}[M |\langle w, \Phi_{d}(X) - T_n(X)
  \rangle_{\mathbb{R}^{d}}] 
  \leq \mathbb{E}[M \sqrt{d} \|\Phi_{d}(X) - T_n(X)\|] \\
  & \leq (1 - n^{-2})*C \delta_{d}^{-3} \sqrt{\frac{ d\log{n}}{n
      \rho_n}} + n^{-2},
  \end{split}
\end{equation}
for some constant $M > 0$. 
Furthermore, we can take $M$ to be
independent of $f$. Thus, there exists some $n_0$ such that for
all $n \geq n_0$, $\sup |R_{\varphi,f} - R_{\varphi,g}| \leq
\epsilon$ where the supremum is taken over all $w \in \mathbb{R}^{d},
\|w\| \leq d$.  
Now let $w^{*}$ be such that $\langle w^{*}, \Phi_{d} \rangle = \arginf_{f \in
  \mathcal{C}^{(d)}_{\Phi}} R_{\varphi,f}$. We then have
  \begin{equation}
    \label{eq:47}
      \inf_{g \in \mathcal{C}^{(d)}_{T_n}} R_{\varphi,g}  \leq
      R_{\varphi, \langle w^{*}, T_n\rangle} \leq R_{\varphi, \langle
        w^{*}, \Phi_{d} \rangle } + \epsilon \leq  \inf_{f \in
        \mathcal{C}^{(d)}_{\Phi}} R_{\varphi,f} + \epsilon 
      \leq R_{\varphi}^{*} + 2\epsilon.
  \end{equation}
If $\varphi$ is a classification-calibrated convex surrogate of
the $0-1$ loss, then there exists a non-decreasing function $\psi
\colon [0,1] \mapsto [0,\infty)$ such that $L(f) - L^{*} \leq
\psi^{-1}(R_{\varphi,f} - R_{\varphi}^{*})$ \citep{bartlett06:_convex}. Thus by Eq.~\eqref{eq:47}
we have
\begin{equation}
  \label{eq:48}
  L^{*}_{T_n} - L^{*} \leq L(\mathrm{sign}(\arginf_{g \in
    \mathcal{C}^{(d)}_{T_n}} R_{\varphi,g})) - L^{*} \leq
  \psi^{-1}(2 \epsilon).
\end{equation}
As $\epsilon$ is arbitrary, the proof is completed.
\end{proof}

\section{Experimental Results}
\label{sec:experimental-results}
\begin{figure*}[htbp] 
\centering \subfloat[scatterplot of the data]{\includegraphics[width=0.5\textwidth]{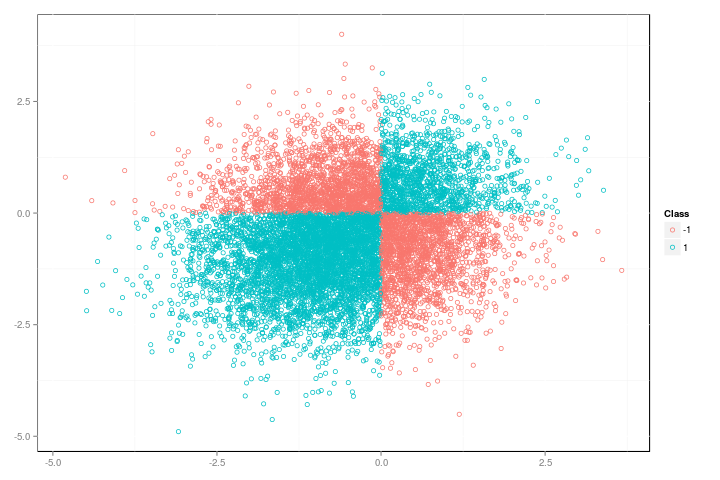}
\label{fig:simulated-a} \hfil }
\centering 
\subfloat[classification performance ]{\includegraphics[width=0.5\textwidth]{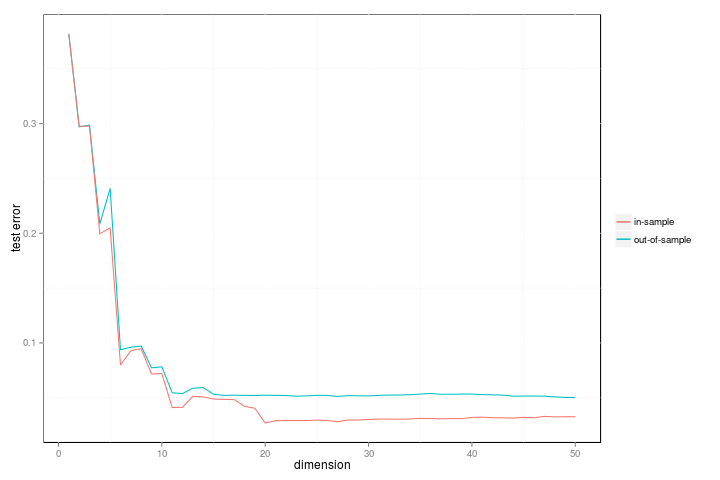}
\label{fig:simulated-b}}
\caption{Comparison of the in-sample against out-of-sample classification
  performance for a simulated data example. The performance
  degradation due to out-of-sample embedding is less than 2\%.}
\label{fig:simulated}
\end{figure*} In this section, we illustrate the out-of-sample
extension described in \S~\ref{sec:experimental-results} by studying
its impact on classification performance through two simulation
examples and a real data example. In our first example, data is
simulated using a mixture of two multivariate normals in
$\mathbb{R}^{2}$. The components of the mixture have equal prior and
the first component of the mixture has mean parameter $(1,1)$ and
identity covariance matrix while the second component has mean
$(-1,-1)$ and identity covariance matrix. We sample $10000$ data
points from this mixture and assign class labels in $\{-1,1\}$ to them
according to the quadrant in which they fall, i.e., if $X_i = (a,b)
\in \mathbb{R}^{2}$ then $Y_i = \mathrm{sign}(ab)$.
Fig~\ref{fig:simulated-a} depicts the scatter plot of the sampled data
colored according to their class labels. The Bayes risk is $0$ for
classifying the $X_i$. A latent position graph $G$ is then generated
based on the sampled data points with $\kappa$ being the Gaussian
kernel.

To measure the in-sample classification performance, we embed $G$ into
$\mathbb{R}^{d}$ for $d$ ranging from $1$ through $50$. A subset of
$2000$ vertices is then selected uniformly at random and designated
as the training data set. The remaining $8000$ vertices constitute
the testing data set. For each choice of
dimension $d$, we select a linear classifier $g_d$ by performing a least
square regression on the $2000$ training data points and measure the
classification error of $g_d$ on the $8000$ testing data points. The
results are plotted in Fig~\ref{fig:simulated-b}.

For the out-of-sample classification performance, we embed the induced
graph $G'$ formed by the $2000$ training vertices in the above
description. For each choice of dimension $d$, we out-of-sample embed
the $8000$ testing vertices into $\mathbb{R}^{d}$. For each choice of
dimension $d$, a linear classifier $g_d$ is once again selected by
linear regression using the in-sample training data points and tested
on the out-of-sample embedded testing data points. The classification
errors are also plotted in Fig~\ref{fig:simulated-b}.  A quick glance
at the plots in Fig~\ref{fig:simulated-b} suggests that the
classification performance degradation due to the out-of-sample
embedding is negligible.

Our next example uses the abalone dataset from the UCI machine
learning repository \citep{bache2013}. The
data set consists of $4177$ observations of nine different abalones
attributes. The attributes are sex, number of rings, and seven other
physical measurements of the abalones, e.g., length, diameter, and
shell weight. The number of rings in an abalone is an estimate of its
age in years. We consider the problem of classifying an abalone based
on its physical measurements. Following the description of the data
set, the class labels are as follows. An abalone
is classified as class 1 if its number of rings is eight or less. It
is classified as class 2 if its number of rings is nine or ten, and it
is classified as class 3 otherwise. The dataset is partitioned into a
training set of $3133$ observations and a test set of $1044$
observations. The lowest misclassification rate
is reported to be 35.39\% \citep{waugh95:_exten_cascad_correl}.

We form a graph $G$ on $4177$ vertices following a latent position
model with a Gaussian kernel $\exp(- 2\|X_i - X_j\|^{2})$ where $X_i
\in \mathbb{R}^{7}$ represents the physical measurements of the $i$-th
abalone observation. To measure the in-sample classification
performance, we embed the vertices of $G$ into $\mathbb{R}^{50}$ and
train a multi-class linear SVM on the embedding of the $3133$ training
vertices. We then measure the mis-classification rate of this
classifier on the embedding of the $1044$ testing vertices. For the
out-of-sample setting, we randomly chose a subset of $m$ vertices from
the training set and embed the
resulting induced subgraph $G_m$ into $\mathbb{R}^{50}$ then
out-of-sample embed the remaining $4177 - m$ vertices. We then train a
multi-class linear SVM on the $3133 - m$ out-of-sample embedded
vertices in the training set and measure the mis-classification error
on the vertices in the testing set. The results for various choices of
$m \in \{200,600,1000,\dots,2600\}$ are given in
Table~\ref{tab:abalone}.
\begin{table}[htbp]
  \centering
  \begin{tabular}[center]{|c|c|c|c|c|c|c|}
    \hline
     $m = 200$ & $m = 600$ & $m = 1000$ & $m = 1400$ & $m = 1800$ & $m
     = 
     2200$ & $m = 2600$ \\ \hline
     0.444 & 0.386 & 0.391 & 0.375 & 0.382 & 0.374 & 0.401 \\ \hline
  \end{tabular}
  \caption{Out-of-sample classification performance for the abalone
    dataset. The in-sample classification performance is $0.358$. The
    lowest reported mis-classification rate is $0.354$. The
    performance degradation due to the out-of-sample embedding is as
    low as $2\%$.}
  \label{tab:abalone}
\end{table}

Our final example is on the CharityNet dataset. The data set consists
of 2 years of anonymized donations transactions between anonymized
donors and charities. There are in total $3.3$ million transactions
representing donations from 1.8 million donors to 5700 charities. Note
that the data set does not contains any explicit information on the
charities to charities relationship, i.e., the charities relate to one
another through the donations transactions between donors and
charities. We investigate the problem of clustering the charities under the
assumption that there are additional information on the donors, but
virtually no information on the charities. 

We can view the problem as embedding an adjacency matrix $\mathbf{A} =
\bigl[\begin{smallmatrix} \mathbf{A}_{dd} & \mathbf{A}_{dc} \\
\mathbf{A}_{dc}^{T} & \mathbf{A}_{cc}
    \end{smallmatrix}\bigr]$
follows by clustering the vertices of $\mathbf{A}$. Here
$\mathbf{A}_{dd}$ represent the (unobserved) donors to donors graph,
$\mathbf{A}_{dc}$ represents the donors to charities graph and
$\mathbf{A}_{cc}$ represents the (unobserved) charities to charities graph. 
Because we only have transactions between donors
and charities, we do not observe any part of $\mathbf{A}$ except
$\mathbf{A}_{dc}$. Using the additional information on the donors,
e.g., geographical information of city and state, we can
simulate $\mathbf{A}_{dd}$ by modeling each of $\mathbf{A}_{dd}(i,j) \sim
\mathrm{Bern}(\exp(-d^2_{ij}))$, where $d_{ij}$ is a pairwise distance
between donors $i$ and $j$. We then use $\mathbf{A}_{dd}$ to obtain
an embedding of the donors. Given this embedding, we out-of-sample
embed the charities and cluster
them using the Mclust implementation of \citep{fraley99:_mclus} for
Gaussian mixture models. We note
that for this example, a biclustering of $\mathbf{A}_{dc}$
is also applicable. 

For this experiment, we randomly sub-sample $10,000$ donors and use
the associated charities and transactions, which yields $1,722$ unique
charities, $9,877$ unique donors, and $17,764$ transactions. There are
$52$ unique states for the charities, and the model-based clustering
yields $\widehat{K}=17$ clusters. We validate our clustering via
calculating the adjusted Rand Index (ARI) between the clustering
labels and the true labels of the charities. We use the state
information of the charities as the true labels, and we obtain an ARI
of $0.01$. This number appears small at first sight so we generate a
null distribution of the adjusted Rand Index by shuffling the true
labels. Figure~\ref{fig:charity} depicts the null distribution of the
ARI with $10,000$ trials. It shows that $\mu=3.5e-05$ and
$\sigma=0.003$. The shaded area indicates the number of times the null
ARIs are larger than the alternative ARI, which is the $p$-value. With
the $p$-value of $6e-04$, we claim that the ARI obtained by clustering
the out-of-sample embedded charities is significantly better than 
chance. In addition, this example also illustrates the applicability
of out-of-sample embedding to scenarios where the lack of information
regarding the relationships between a subset of the rows might prevents the use of
spectral decomposition algorithms for embedding the whole matrix. \begin{figure*}[htbp] 
\centering 
\subfloat[Density plot]{
\includegraphics[width=0.45\textwidth]{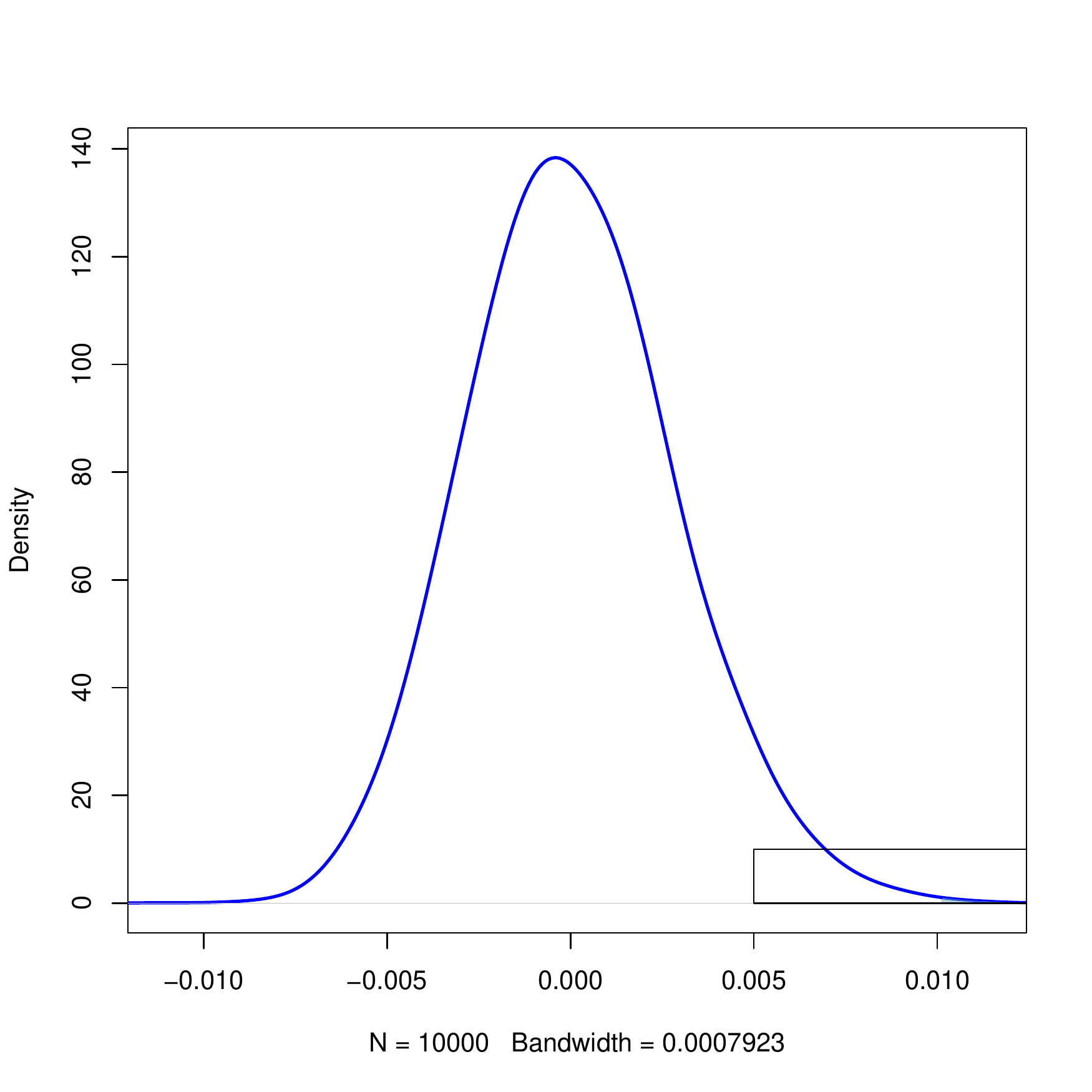}
\hfil
}
\subfloat[Zoomed-in tail of density plot]{\includegraphics[width=0.45\textwidth]{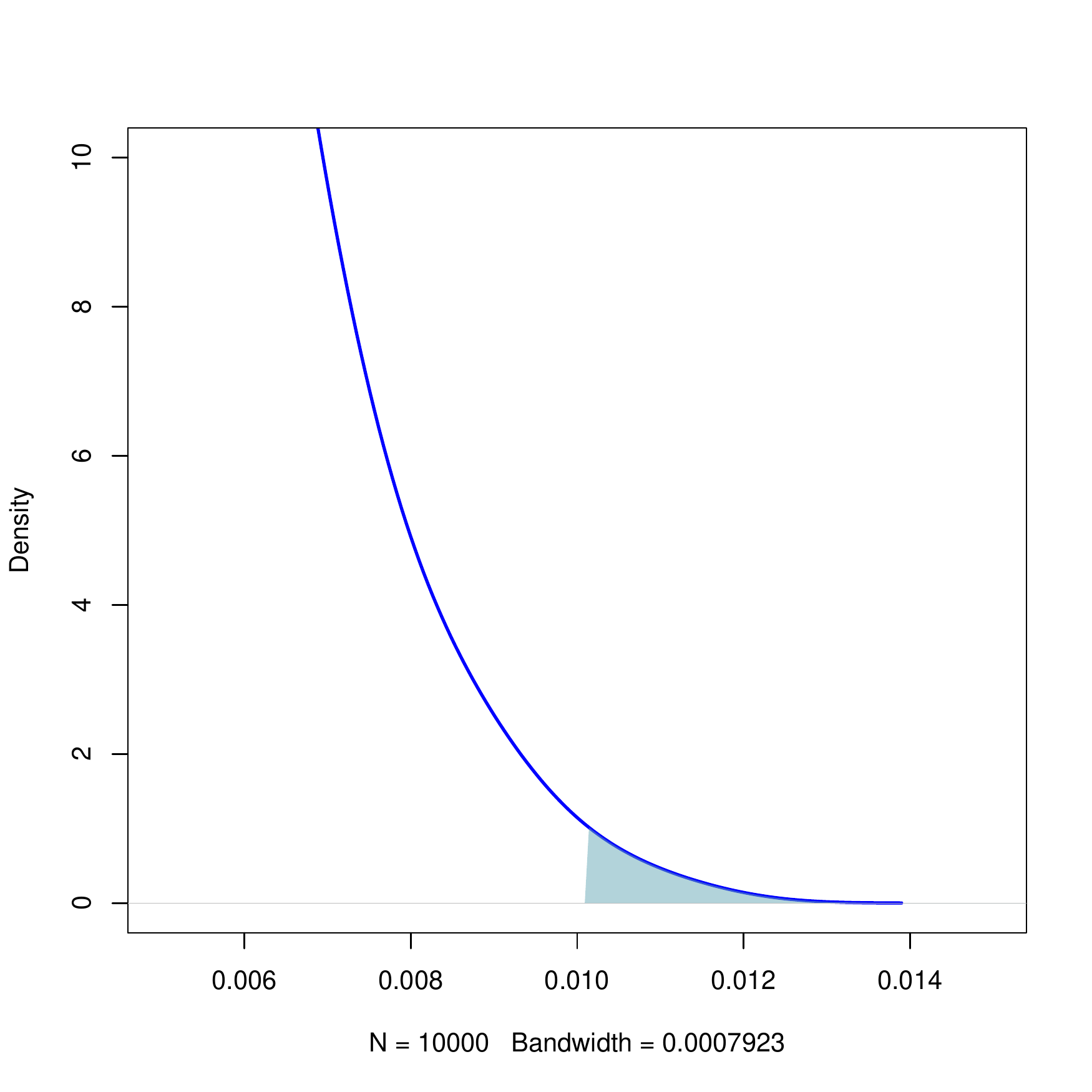}}
\caption{Density plots for the null distribution, under a
  permutation test, of the ARI values
between the clustering labels and the permuted true labels (state
information of the charities). The shaded area above the ARI value
between the clustering labels and the true labels represent the
estimated $p$-value. The plot indicates
that the ARI value of the clustering of the out-of-sample charities
is (statistically significant) better than chance.}
\label{fig:charity}
\end{figure*} 

\section{Conclusions}
\label{sec:conclusions} In this paper we investigated the
out-of-sample extension for embedding out-of-sample vertices in graphs
arising from a latent position model with positive definite kernel
$\kappa$. We showed, in Theorem~\ref{thm:2}, that if the number of
in-sample vertices is sufficiently large, then with high-probability,
the embedding into $\mathbb{R}^{d}$ given by the out-of-sample
extension is close to the true (truncated) feature map $\Phi_d$. This
implies that inference for the out-of-sample vertices using their
embeddings is appropriate, e.g.,Theorem~\ref{thm:1}. Experimental
results on simulated data suggest that under suitable conditions, the
degradation due to the out-of-sample extension is negligible. 

The out-of-sample extension described in this paper is related to the
notion of ``sketching'' and Nystr\"{o}m approximation for matrices
\citep{bengio04:_out_lle_isomap_mds_eigen,williams01:_using_nystr,gittens:_revis_nystr,drineas05:_nystr_gram,platt05:_fastm_metric_mds_nyst}. This
connection suggests inquiry on how to select the in-sample vertices
via consideration of the Nystr\"{o}m approximation so as to yield the
best inferential performance on the out-of-sample vertices, i.e.,
whether one can use results on error bounds in the Nystr\"{o}m
approximation to augment the selection of the in-sample vertices. A
possible approach might be to select the sketching matrix
$\mathbf{S}$, and hence the in-sample vertices, via a non-uniform
importance sampling based on the leverage scores of the rows of
$\mathbf{A}$. The leverage score of row $i$ of $\mathbf{A}$ is the
$l_2$ norm of the $i$-th row of $\mathbf{U}$ in the
eigen-decomposition $\mathbf{U} \bm{\Sigma} \mathbf{U}^{T}$ of
$\mathbf{A}$, and fast approximation methods to compute the leverage
scores are available, see e.g.
\citet{clarkson13:_low,drineas12:_fast}. We believe the investigation of this and
other approaches to selecting the in-sample vertices will yield 
results that are useful and relevant for application domains.

Finally, as mentioned in Section~\ref{sec:out-sample-estim}, the
out-of-sample extension as defined in this paper depends only on the
in-sample vertices. Hence, the embedding of a batch of out-of-sample
vertices does not uses the information contained in the relationship
between the out-of-sample vertices. A modification of the
out-of-sample extension presented herein that uses this
information in the batch setting is possible, see e.g.
\cite{trosset08} for such a modification in the case of classical
multidimensional scaling. However, the construction similar to that in
\cite{trosset08} will yield a convex but non-linear optimization problem
with no closed-form solution and is much more complicated to
analyze. We thus note that it is of potential interest to introduce an
out-of-sample extension in the batch setting that is simple and amenable to
analysis. 
\section*{Acknowledgements} 
This work was partially supported by
National Security Science and Engineering Faculty Fellowship (NSSEFF),
Johns Hopkins University Human Language Technology Center of
Excellence (JHU HLT COE), the XDATA program of the Defense
Advanced Research Projects Agency (DARPA) administered through Air
Force Research Laboratory contract FA8750-12-2-0303, and
the Acheson J. Duncan Fund for the Advancement of Research in
Statistics. 
\appendix
\section*{Appendix A: Proof of Theorem~\ref{thm:2}}
\label{sec:proof-theor-refthm:2} 
We now proceed to prove Theorem~\ref{thm:2}. 
First recall the definition of $T_n(X)$ in terms of the Moore-Penrose
pseudo-inverse $\mathbf{Z}^{\dagger}$ of $\mathbf{Z} =  \mathbf{U}_{\mathbf{A}}
\mathbf{S}_{\mathbf{A}}^{1/2}$ and $\bm{\xi}$ from Definition~\ref{def:1}. We
consider the expression
\begin{equation}
  \label{eq:53}
  T_n(X) = \mathbf{Z}^{\dagger} \bm{\xi} =
  (\mathbf{Z} - \mathbf{W} \tilde{\mathbf{Z}}^{\dagger}) \bm{\xi} +
  \mathbf{W} \tilde{\mathbf{Z}}^{\dagger}(\bm{\xi} - \mathbb{E}[\bm{\xi}]) +
  \mathbf{W} \tilde{\mathbf{Z}}^{\dagger} \mathbb{E}[\bm{\xi}]
\end{equation} where $\tilde{\mathbf{Z}}^{\dagger}$ is the
Moore-Penrose pseudo-inverse of $\tilde{\mathbf{Z}} =
\mathbf{U}_{\mathbf{K}} \mathbf{S}_{\mathbf{K}}^{1/2}$ and
$\mathbf{W}$ is some orthogonal matrix in $\mathcal{M}_{d}$. A rough
sketch of the argument then goes as follows. We first show that
$\mathbf{Z}^{\dagger}$ is ``close'' (up to rotation) in operator norm
to $\tilde{\mathbf{Z}}^{\dagger}$. This allows us to conclude that
$(\mathbf{Z}^{\dagger} - \tilde{\mathbf{Z}}^{\dagger}) \bm{\xi}$ is
``small''. We then show that $\tilde{\mathbf{Z}}^{\dagger} (\bm{\xi} -
\mathbb{E}[\bm{\xi}])$ is ``small'' as it is a sum of zero-mean random
vectors in $\mathbb{R}^{d}$. We then relate $\rho_{n}^{-1/2}
\tilde{\mathbf{Z}}^{\dagger} \mathbb{E}[\bm{\xi}]$ to the projection
$\hat{\mathcal{P}}_{d}$ of the feature map $\Phi$ into
$\mathbb{R}^{d}$ where $\hat{\mathcal{P}}_{d}$ is induced by the
eigenvectors of $\mathbf{K}$. Finally, we use results on the
convergence of spectra of $\mathbf{K}$ to the spectra of $\mathcal{K}$
to show that the projection $\hat{\mathcal{P}}_{d}$ of $\Phi$ is
``close'' (up to rotation) to the projection that maps $\Phi$ into
$\Phi_{d}$. We thus arrive at an expression of the form
$\rho_{n}^{-1/2} \mathbf{W} T_n(X) \approx \Phi_{d}(X)$ as in the statement of
Theorem~\ref{thm:2}.

We first collect some assorted bounds for the eigenvalues of
$\mathbf{A}$ and $\mathbf{K}$ and bounds for the projection onto the
subspaces of $\mathbf{A}$ or $\mathbf{K}$ in the following
proposition. 
\begin{proposition}
  \label{prop:1} Let $\mathcal{P}_{\mathbf{A}}$ and
$\mathcal{P}_{\mathbf{K}}$ be the projection operators onto the
subspace spanned by the eigenvectors corresponding to the $d$ largest
eigenvalues of $\mathbf{A}$ and $\mathbf{K}$, respectively. Denote by
$\delta_{d}$ the quantity $\lambda_{d}(\mathcal{K}) -
\lambda_{d+1}(\mathcal{K})$ and suppose that $\delta_{d} > 0$. Assume
also that $n$ satisfies $\delta_{d}(\mathcal{K}) > 4\sqrt{2}\sqrt{(n \rho_n)^{-1}
\log{(n/\eta)}}$. Then with probability at least $1 - 2\eta$, the
following inequalities hold simultaneously.
    \begin{gather}
      \| \mathbf{A} - \mathbf{K} \| \leq 2 \sqrt{ n \rho_n
        \log{(n/\eta)}} \label{eq:1} \\
      \lambda_{1}(\mathbf{A}) \leq n \rho_n
      ; \qquad \lambda_{1}(\mathbf{K})
      \leq  n \rho_n \\
      \lambda_{d}(\mathbf{A}) \geq n \rho_n
      \lambda_{d}(\mathcal{K})/2; \qquad \lambda_{d}(\mathbf{K}) \geq
      n \rho_n \lambda_{d}(\mathcal{K})/2 \\
      \| \mathcal{P}_{\mathbf{A}} - \mathcal{P}_{\mathbf{K}} \| \leq 4
      \delta_{d}^{-1} \sqrt{ n \rho_{n} \log{(n/\eta)}} \label{eq:18} \\ 
      \| (\mathcal{P}_{\mathbf{A}} \mathbf{A}) -
  (\mathcal{P}_{\mathbf{K}} \mathbf{K}) \| \leq 6 \delta_{d}^{-1}
  \sqrt{n \rho_n \log{(n/\eta)}}\label{eq:50}
    \end{gather}
\end{proposition}
\begin{proof}[Sketch] Eq.~\eqref{eq:1} is from
\cite{oliveira2010concentration}. The bound for
$\lambda_{1}(\mathbf{K})$ follows from the assumption that the range
of $\kappa$ is in $[0,1]$. The bound for $\lambda_{d}(\mathbf{K})$
follows from Theorem~\ref{thm:5} below. The bounds for the eigenvalues
of $\mathbf{A}$ follow from the bounds for the corresponding
eigenvalues of $\mathbf{K}$, Eq.~\eqref{eq:1}, and 
perturbation results, e.g. Corollary III.2.6 in
\citet{bhatia97:_matrix_analy}. Eq.~\eqref{eq:18} follows from 
Eq.~\eqref{eq:1} and the $\sin \bm{\Theta}$ theorem \citep{davis70}. Eq.~\eqref{eq:50} follows
from Eq.~\eqref{eq:18}, Eq.~\eqref{eq:1}, and an application of the
triangle inequality.
\end{proof}
We also note the following result on perturbation for pseudo-inverses from
\cite{wedin73:_pertub_pseud}.
\begin{lemma}
  \label{lem:2}
  Let $\mathbf{A}$ and $\mathbf{B}$ be
  matrices with $\mathrm{rk}(\mathbf{A}) =
  \mathrm{rk}(\mathbf{B})$. Let $\mathbf{A}^{\dagger}$ and
  $\mathbf{B}^{\dagger}$ be the Moore-Penrose pseudo-inverses of
  $\mathbf{A}$ and $\mathbf{B}$, respectively. Then
  \begin{equation}
    \label{eq:11}
    \| \mathbf{A}^{\dagger} - \mathbf{B}^{\dagger} \| \leq \frac{(1 +
      \sqrt{5})}{2} (\| \mathbf{A}^{\dagger}
    \|)(\|\mathbf{B}^{\dagger}\|)\|\mathbf{A} - \mathbf{B}\| 
  \end{equation}
\end{lemma}
We now provide a bound for the spectral norm of the difference
$\mathbf{Z}^{\dagger} - \tilde{\mathbf{Z}}^{\dagger}$.  
\begin{lemma}
  \label{lem:1}
  Let $\mathbf{Z} = \mathbf{U}_{\mathbf{A}}
  \mathbf{S}_{\mathbf{A}}^{1/2}$ and $\tilde{\mathbf{Z}} =
  \mathbf{U}_{\mathbf{K}} \mathbf{S}_{\mathbf{K}}^{1/2}$. Then, with probability at least, $1 - 2\eta$, there
  exists an orthogonal $\mathbf{W} \in \mathcal{M}_{d}$ such that  
  \begin{equation}
    \label{eq:5}
    \|\mathbf{W} \mathbf{Z}^{\dagger} - \tilde{\mathbf{Z}}^{\dagger} \|
    \leq 24(1 + \sqrt{5}) \frac{\sqrt{\log{(n/\eta)}}}{n \rho_n
        \delta_d^{3}} 
\end{equation}
\end{lemma} 
\begin{proof}
We have $\mathbf{Z}^{\dagger}
(\mathbf{Z}^{\dagger})^{T} = \mathbf{S}_{\mathbf{A}}^{-1}$ 
and $\tilde{\mathbf{Z}}^{\dagger} (\tilde{\mathbf{Z}}^{\dagger})^{T}
=  \mathbf{S}_{\mathbf{K}}^{-1}$. Thus,
$\|\mathbf{Z}^{\dagger}\| = \| \mathbf{S}_{\mathbf{A}}^{-1} \|^{1/2} =
(\lambda_{d}(\mathbf{A}))^{-1/2}$ and $\|
\tilde{\mathbf{Z}}^{\dagger}\| = \|
\mathbf{S}_{\mathbf{K}}^{-1}\|^{1/2} =
(\lambda_{d}(\mathbf{K}))^{-1/2}$. 
Then by Lemma~\ref{lem:2}, we have
  \begin{equation*}
    \label{eq:6}
    \| \mathbf{W} \mathbf{Z}^{\dagger} - \tilde{\mathbf{Z}}^{\dagger}
    \| \leq \frac{1+\sqrt{5}}{2} (\| \mathbf{S}_{\mathbf{A}}^{-1} \|^{1/2}) (\|
    \mathbf{S}_{\mathbf{K}}^{-1} \|^{1/2}) \| \mathbf{Z} \mathbf{W}^{T} -
    \tilde{\mathbf{Z}} \|
\end{equation*}
  for any orthogonal $\mathbf{W} \in \mathcal{M}_{d}$. By Proposition~\ref{prop:1}, with probability at least $1 -
  2\eta$,
  \begin{equation*}
    \begin{split}
    (\| \mathbf{S}_{\mathbf{A}}^{-1} \|^{1/2}) (\| \mathbf{S}_{\mathbf{K}}^{-1} \|^{1/2}) =
    1/\sqrt{\lambda_{d}(\mathbf{A})\lambda_{d}(\mathbf{K})} \leq
    2/(n\rho_n \lambda_{d}(\mathcal{K})).
\end{split}
\end{equation*} 
To complete the proof, we show that with probability at least $1 -
2\eta$, there exists some orthogonal $\mathbf{W} \in \mathcal{M}_{d}$
such that
  \begin{equation*}
    \label{eq:14}
 \| \mathbf{Z} \mathbf{W}^{T} -
    \tilde{\mathbf{Z}} \| \leq 24 \delta_{d}^{-2} \sqrt{\log{(n/\eta)}}.
  \end{equation*} 
We proceed as follows. We note that $\mathbf{Z}$ and
$\tilde{\mathbf{Z}}$ are matrices in $\mathcal{M}_{n,d}$ and are of
full column rank. Then by Lemma A.1 in \cite{tangs.:_univer}, there exists
an orthogonal matrix $\mathbf{W} \in \mathcal{M}_{d}$ such that
\begin{equation*}
  \| \mathbf{Z} \mathbf{W}^{T} - \tilde{\mathbf{Z}} \| \leq
    \|\mathbf{Z} \mathbf{Z}^{T} - \tilde{\mathbf{Z}}
    \tilde{\mathbf{Z}}^{T} \| \frac{\sqrt{\|\vphantom{\tilde{\mathbf{Z}}^{T}}\mathbf{Z}
        \mathbf{Z}^{T}\|} + \sqrt{\|\tilde{\mathbf{Z}}
        \tilde{\mathbf{Z}}^{T}\|}}{\lambda_{d}(\tilde{\mathbf{Z}}
        \tilde{\mathbf{Z}}^{T})}
\end{equation*}
As $\mathbf{Z}
\mathbf{Z}^{T} = \mathbf{U}_{\mathbf{A}} \mathbf{S}_{\mathbf{A}}
\mathbf{U}_{\mathbf{A}} = (\mathcal{P}_{\mathbf{A}} \mathbf{A})$ and
$\tilde{\mathbf{Z}} \tilde{\mathbf{Z}}^{T} = (\mathcal{P}_{\mathbf{K}}
\mathbf{K})$, we thus have
\begin{equation}
  \begin{split}
    \label{eq:52}
  \| \mathbf{Z} \mathbf{W}^{T} - \tilde{\mathbf{Z}} \| & \leq \|
  (\mathcal{P}_{\mathbf{A}} \mathbf{A}) - (\mathcal{P}_{\mathbf{K}}
\mathbf{K}) \| \frac{\sqrt{\lambda_{1}(\mathbf{A})} +
  \sqrt{\lambda_{1}(\mathbf{K})}}{\lambda_{d}(\mathbf{K})} \\  & \leq
6 \delta_{d}^{-1} \sqrt{n \rho_{n} \log{(n/\eta)}} \frac{2\sqrt{n
    \rho_{n} }}{ n \rho_{n} \delta_{d}/2} \\ &\leq 24 \delta_{d}^{-2} \sqrt{\log{(n/\eta)}}
\end{split}
\end{equation}
where the inequalities in Eq.~\eqref{eq:52} follows from
Proposition~\ref{prop:1} and hold with probability $1 -
2\eta$. Eq.~\eqref{eq:5} is thus established. 
\end{proof}
We now provide a bound for $\| (\mathbf{W} \mathbf{Z}^{\dagger} -
\tilde{\mathbf{Z}}^{\dagger}) \bm{\xi} \|$ which, as sketched earlier,
is one of the key step in the proof of Theorem~\ref{thm:2}. We note
that application of the multiplicative bound for the norm of a
matrix vector product, i.e., $\|(\mathbf{W} \mathbf{Z}^{\dagger} -
\tilde{\mathbf{Z}}^{\dagger}) \bm{\xi}\| \leq (\|(\mathbf{W} \mathbf{Z}^{\dagger} -
\tilde{\mathbf{Z}}^{\dagger}) \|)(\|\bm{\xi}\|)$ leads to a bound
that is worse by a factor of $\rho_{n}^{-1/2}$. This is due to the
fact that $\bm{\xi}$ is a vector whose components are independent Bernoulli random
variables and thus the scaling of the probabilities
$\mathbb{E}[\bm{\xi}]$ by a constant $c$ changes  $\|\bm{\xi}\|$ by a
  factor that is roughly $c^{1/2}$.  
\begin{lemma}
  \label{lem:5}
  With probability at least $ 1 - 2\eta$, there exists an orthogonal
  $\mathbf{W} \in \mathcal{M}_d$ such that
  \begin{equation}
    \label{eq:51}
    \| (\mathbf{W} \mathbf{Z}^{\dagger} -
    \tilde{\mathbf{Z}}^{\dagger}) \bm{\xi} \| \leq C \delta_{d}^{-3}
    \sqrt{\frac{\log{(n/\eta)}}{n}}
  \end{equation}
\end{lemma}
The proof of Lemma~\ref{lem:5} uses the following concentration inequality for sums of independent
matrices from \citet{tropp12:_user}.  
\begin{theorem}
  \label{thm:4} Consider a finite sequence $\mathbf{B}_{k}$ of
independent random matrices with dimensions $d_1 \times d_2$. Assume
that each $\mathbf{B}_k$ satisfies $\mathbb{E}[\mathbf{B}_k] = \bm{0}$
and that, for some $R \geq 0$ independent of the $\mathbf{B}_K$,
$\|\mathbf{B}_{k}\| \leq R$ almost surely. Define
\begin{equation*}
  \sigma^{2} :=  \max\Bigl\{ \Bigl\| \sum_{k} \mathbb{E}[\mathbf{B}_k \mathbf{B}_{k}^{*}]
  \Bigr\|, \Bigl\| \sum_{k} \mathbb{E}[\mathbf{B}_{k}^{*} \mathbf{B}_k] \Bigr\|\Bigr\}
\end{equation*}
Then, for all $t \geq 0$, one has 
\begin{equation}
  \label{eq:24}
  \mathbb{P}\Bigl[ \Bigl\| \sum_{k} \mathbf{B}_k \Bigr\| \geq t \Bigr] \leq (d_1
  + d_2) \exp \Bigl(\frac{-t^2/2}{\sigma^{2} + R t/3}\Bigr).
\end{equation}
\end{theorem}
\begin{proof}[Lemma~\ref{lem:5}]
  Let $\mathbf{W}$ be the matrix that minimizes $\|
  \mathbf{W} \mathbf{Z}^{\dagger} - \tilde{\mathbf{Z}}^{\dagger}\|$
  over the set of orthogonal matrices. Let $\bm{b}_{i}$ be the $i$-th
    column of $\mathbf{W} \mathbf{Z}^{\dagger} -
    \tilde{\mathbf{Z}}^{\dagger}$. We have
    \begin{equation*}
      \begin{split}
      \|(\mathbf{W} \mathbf{Z}^{\dagger} -
    \tilde{\mathbf{Z}}^{\dagger}) \bm{\xi}\| &= \|(\mathbf{W} \mathbf{Z}^{\dagger} -
    \tilde{\mathbf{Z}}^{\dagger}) \mathbb{E}[\bm{\xi}]\| + \|(\mathbf{W} \mathbf{Z}^{\dagger} -
    \tilde{\mathbf{Z}}^{\dagger}) (\bm{\xi} - \mathbb{E}[\bm{\xi}])\|
    \\ &= \|(\mathbf{W} \mathbf{Z}^{\dagger} -
    \tilde{\mathbf{Z}}^{\dagger}) \mathbb{E}[\bm{\xi}]\| +
    \Bigl\|\sum_{i=1}^{n} \bm{b}_{i} (\xi_i - \mathbb{E}[\xi_i])\Bigr\|
     \end{split} 
    \end{equation*}
    By Lemma~\ref{lem:1}, we have
    \begin{equation*}
      \|(\mathbf{W} \mathbf{Z}^{\dagger} -
    \tilde{\mathbf{Z}}^{\dagger}) \mathbb{E}[\bm{\xi}]\| \leq (\|\mathbf{W} \mathbf{Z}^{\dagger} -
    \tilde{\mathbf{Z}}^{\dagger}\|)\|\mathbb{E}[\bm{\xi}]\| \leq 24(1 + \sqrt{5}) \delta_{d}^{-3} \sqrt{\frac{\log{(n/\eta)}}{n}}
    \end{equation*}
    with probability at least $1 - 2\eta$. We now apply
    Theorem~\ref{thm:4} to the term
    $\sum_{i=1}^{n} \bm{b}_{i} (\xi_i - \mathbb{E}[\xi_i])$. We note
    that
    \begin{equation*}
      \begin{split}
      \sum_{i=1}^{n} \mathbb{E}[ \bm{b}_i \bm{b}_i^{T} (\xi_i -
      \mathbb{E}[\xi_i])^2] &= \sum_{i=1}^{n} \bm{b}_i \bm{b}_i^{T} \rho_n
      \kappa(X,X_i) (1 - \rho_n \kappa(X,X_i)) \\ & \prec \sum_{i=1}^{n}
      \rho_n \bm{b}_{i} \bm{b}_{i}^{T} \\ & \prec \rho_{n} (\mathbf{W}
      \mathbf{Z}^{\dagger} - \tilde{\mathbf{Z}}^{\dagger})(\mathbf{W}
      \mathbf{Z}^{\dagger} - \tilde{\mathbf{Z}}^{\dagger})^{T}  
      \end{split}
    \end{equation*}
    where $\prec$ refers to the positive semidefinite ordering for
    matrices. Similarly, we have
    \begin{equation*}
      \begin{split}
      \sum_{i=1}^{n} \mathbb{E}[ \bm{b}_i^{T} \bm{b}_i (\xi_i -
      \mathbb{E}[\xi_i])^2] \leq \rho_{n} \sum_{i=1}^{n}
       \bm{b}_{i}^{T} \bm{b}_{i} = \rho_{n} \mathrm{Tr} \Bigl[(\mathbf{W}
      \mathbf{Z}^{\dagger} - \tilde{\mathbf{Z}}^{\dagger})^{T} (\mathbf{W}
      \mathbf{Z}^{\dagger} - \tilde{\mathbf{Z}}^{\dagger})\Bigr]
      \end{split}
    \end{equation*}
We thus have
    \begin{equation*}
      \sigma^{2}  \leq \rho_n \|(\mathbf{W}
      \mathbf{Z}^{\dagger} - \tilde{\mathbf{Z}}^{\dagger}) \|_{F}^{2}
    \end{equation*}
Theorem~\ref{thm:4} now applies to give
\begin{equation*}
  \| (\mathbf{W} \mathbf{Z}^{\dagger} -
    \tilde{\mathbf{Z}}^{\dagger}) (\bm{\xi} - \mathbb{E}[\bm{\xi}]) \| \leq \sqrt{2 d  
      \rho_{n} \log{(n/\eta)}} \|\mathbf{W} \mathbf{Z}^{\dagger} -
    \tilde{\mathbf{Z}}^{\dagger} \|_{F}
\end{equation*}
with probability at least $1 - 2\eta$. We therefore have
\begin{equation*}
  \|(\mathbf{W} \mathbf{Z}^{\dagger} -
    \tilde{\mathbf{Z}}^{\dagger}) \bm{\xi}\| \leq C_1 \delta_{d}^{-3}
    \sqrt{\frac{\log{(n/\eta)}}{n}} + C_2 \delta_{d}^{-3} 
    \frac{d \log{(n/\eta)}}{n \sqrt{\rho_n}} 
\end{equation*}
Under our assumption of $n \rho_{n} = o(1)$, the above bound
simplifies to Eq.~\eqref{eq:51} as desired.  
\end{proof}
Theorem~\ref{thm:4} also allows us to bound the term $\mathbf{W}_{2}
\tilde{\mathbf{Z}}^{\dagger}(\bm{\xi} - \bm{p})$ with a bound of the
form $C \delta_{d}^{-3} n^{-1/2} \sqrt{ \log{(n/\eta)}}$. Thus, the last
key step of the proof is to relate $\tilde{\mathbf{Z}}^{\dagger} \mathbb{E}[\bm{\xi}] \in
\mathbb{R}^{d}$ to the truncated feature map $\Phi_d$. This will be
done by relating the eigenvalues and eigenvectors of $\mathbf{K}$ to
the eigenvalues and eigenfunctions of $\mathcal{K}$. But as
$\mathbf{K}$ is an operator on $\mathbb{R}^{n}$ and $\mathcal{K}$ is
an operator on $L^{2}(\mathcal{X}, F)$, we will relate these
eigenvalues, eigenvectors and eigenfunctions through some auxiliary
operators on the reproducing kernel Hilbert space $\mathcal{H}$ of
$\kappa$. Following \cite{rosasco10:_integ_operat}, we introduce the
operators $\mathscr{K}_{\mathscr{H}} \colon \mathscr{H} \mapsto \mathscr{H}$ and
$\mathscr{K}_{\mathscr{H},n} \colon \mathscr{H} \mapsto \mathscr{H}$
defined by
 \begin{gather*}
  \mathscr{K}_{\mathscr{H}} \eta = \int_{\mathcal{X}} \langle \eta,
  \kappa(\cdot, x)
  \rangle_{\mathscr{H}} \kappa(\cdot, x) dF(x) \\
  \mathscr{K}_{\mathscr{H},n} \eta = \frac{1}{n} \sum_{i=1}^{n}
  \langle \eta, \kappa(\cdot, {X_i}) \rangle_{\mathscr{H}}
  \kappa(\cdot, {X_i}).
\end{gather*} The operator $\mathscr{K}_{\mathscr{H}}$ and
$\mathscr{K}_{\mathscr{H},n}$ are defined on the same Hilbert space
$\mathscr{H}$, in contrast to $\mathscr{K}$ and $\mathbf{K}$ which are
defined on the different spaces $L^{2}(\mathcal{X},F)$ and
$\mathbb{R}^{n}$, and we can relate $\mathcal{K}_{\mathscr{H},n}$ to
$\mathcal{K}_{\mathscr{H}}$ (Theorem~\ref{thm:5}). In addition, we can relate the
eigenvalues and eigenfunctions of $\mathscr{K}$ to that of
$\mathscr{K}_{\mathscr{H}}$ as well as the eigenvalues and
eigenvectors of $\mathbf{K}$ to the eigenvalues and eigenfunctions of
$\mathscr{K}_{\mathscr{H},n}$, therefore giving us a relationship
between the eigenvalues/eigenfunctions of $\mathscr{K}$ and the
eigenvalues/eigenvectors of $\mathbf{K}$. A precise statement of the
relationships is contained in Proposition~\ref{prop:5} and
Theorem~\ref{thm:5}.
\begin{proposition}[\cite{rosasco10:_integ_operat},\cite{luxburg08:_consis}]
  \label{prop:5}
  The operators $\mathscr{K}_{\mathscr{H}}$ and
  $\mathscr{K}_{\mathscr{H},n}$ are positive, self-adjoint operators
  and are of trace class with $\mathscr{K}_{\mathscr{H},n}$ being of
  finite rank. The spectra of $\mathscr{K}$ and
  $\mathscr{K}_{\mathscr{H}}$ are contained in $[0,1]$ and are the
  same, possibly up to the zero eigenvalues. If $\lambda$ is a
  non-zero eigenvalue of $\mathscr{K}$ and $u$ and $v$ are associated
  eigenfunction of $\mathscr{K}$ and $\mathscr{K}_{\mathscr{H}}$,
  normalized to norm $1$ in $L^{2}(\mathcal{X}, F)$ and $\mathscr{H}$, 
  respectively, then
  \begin{equation}
    \label{eq:28}
    u(x) = \frac{v(x)}{\sqrt{\lambda}}; \quad v(x) = \frac{1}{\sqrt{\lambda}}
    \int_{\mathcal{X}} \kappa(x,x') u(x') dF(x')
  \end{equation}
  Similarly, the spectra of
  $\mathbf{K}/(n \rho_n)$ and $\mathscr{K}_{\mathscr{H},n}$ are contained in
  $[0,1]$ and are the same, possibly up to the zero eigenvalues. If
  $\hat{\lambda}$ is a non-zero eigenvalue of $\mathbf{K}/(n\rho_n)$ and
  $\hat{u}$ and $\hat{v}$ are the corresponding eigenvector and
  eigenfunction of $\mathbf{K}/(n\rho_n)$ and
  $\mathscr{K}_{\mathscr{H},n}$, normalized to norm $1$ in
  $\mathbb{R}^{n}$ and $\mathscr{H}$, respectively, then
  \begin{equation}
    \label{eq:36}
    \hat{u}_{i} = \frac{\hat{v}(x_i)}{\sqrt{\hat{\lambda}}}; \quad 
    \hat{v}(\cdot) = \frac{1}{\sqrt{\hat{\lambda}n}} \sum_{i=1}^{n}
    \kappa(\cdot ,x_i) \hat{u}_{i}
  \end{equation}
\end{proposition} 
Eq.~\eqref{eq:36} in Proposition~\ref{prop:5} states
that an eigenvector $\hat{u}$ of $\mathbf{K}/(n \rho_n)$, which is only defined
for $X_1, X_2, \dots, X_n$, can be extended to an eigenfunction
$\hat{v} \in \mathscr{H}$ of $\mathscr{K}_{\mathscr{H},n}$ defined for
all $x \in \mathcal{X}$, and furthermore, that $\hat{u}_i =
\hat{v}(X_i)$ for all $i=1,2,\dots,n$. 
\begin{theorem}[\cite{rosasco10:_integ_operat,zwald06}]
  \label{thm:5}
    Let $\eta > 0$ be arbitrary. Then with probability at least $1 -
  2e^{-\eta}$,
  \begin{equation}
    \label{eq:29}
    \| \mathscr{K}_{\mathscr{H}} - \mathscr{K}_{\mathscr{H},n} \|_{HS}
    \leq 2 \sqrt{2} \sqrt{\frac{\eta}{n}}
  \end{equation}
  where $\| \cdot \|_{HS}$ is the Hilbert-Schmidt norm. 
  Let $\delta_{d} = \lambda_{d}(\mathcal{K}) -
  \lambda_{d+1}(\mathcal{K})$. 
  For a given $d \geq 1$
  and an arbitrary $\eta > 0$, if the number $n$ of samples $X_i \sim F$ satisfies
  \begin{equation*}
    4\sqrt{2} \sqrt{\frac{\eta}{n}} < \delta_{d}
  \end{equation*}
  then with probability greater than $1 - 2e^{-\eta}$
  \begin{equation}
    \label{eq:32}
    \| \mathcal{P}_{d} - \hat{\mathcal{P}}_{d} \|_{HS} \leq
    \frac{2\sqrt{2} \sqrt{\eta}}{\delta_{d} \sqrt{n}}
  \end{equation}
  where $\mathcal{P}_{d}$ is the projection onto the subspace spanned
  by the eigenfunctions corresponding to the $d$ largest eigenvalues
  of $\mathscr{K}$ and $\hat{\mathcal{P}}_{d}$ is the projection onto
  the subspace spanned by the eigenfunctions corresponding to the $d$
  largest eigenvalues of $\mathscr{K}_{\mathscr{H},n}$. 
\end{theorem}
With the above technical details in place, the following result states
that $ \rho_n^{-1/2}\tilde{\mathbf{Z}}^{\dagger} \mathbb{E}[\bm{\xi}]$
is equivalent to the isometric embedding of $\hat{\mathcal{P}}_{d}
\kappa(\cdot, X)$ into $\mathbb{R}^{d}$.
\begin{lemma}
  \label{lem:4}
  Let $\hat{\mathcal{P}}_{d}$ be the projection onto the
subspace spanned by the eigenfunctions corresponding to the $d$
largest eigenvalues of $\mathscr{K}_{\mathscr{H},n}$. Then
\begin{equation}
  \label{eq:4}
\rho_{n}^{-1/2} \tilde{\mathbf{Z}}^{\dagger} \mathbb{E}[\bm{\xi}] = \imath( \hat{\mathcal{P}}_{d}
\kappa(\cdot, X))
\end{equation}
where $\imath$ is the isometric isomorphism of the finite-dimensional subspace
corresponding to the projection $\hat{\mathcal{P}}_{d}$ into $\mathbb{R}^{d}$. 
\end{lemma}
\begin{proof}
  Let $\{\lambda_{r}\}$ and $\{\psi_{r}\}$ be
  the eigenvalues and eigenfunctions of $\mathcal{K}$. Let
  $\hat{\lambda}_{1}, \hat{\lambda}_{2}, \dots, \hat{\lambda}_{d}$ be
  the eigenvalues of $\mathbf{K}/(n \rho_n)$ and let  
  $\hat{u}^{(1)}, \hat{u}^{(2)}, \dots, \hat{u}^{(d)}$ be the associated
  eigenvectors, normalized to have norm $1$ in
  $\mathbb{R}^{n}$. Also let $\hat{v}^{(1)}, \hat{v}^{(2)}, \dots,
  \hat{v}^{(d)}$ be the corresponding eigenfunctions of
  $\mathcal{K}_{\mathscr{H},n}$, normalized to have norm $1$ in
  $\mathscr{H}$. 
  Let $\Psi_{r,n} = (\sqrt{\lambda}_{r} \psi_{r}(X_i))_{i=1}^{n} \in
  \mathbb{R}^{n}$. Then the $s$-th component of
  $\mathbf{U}_{\mathbf{K}}^{T} \mathbb{E}[\bm{\xi}] \in
  \mathbb{R}^{d}$ for $s = 1,2,\dots,d$ is of the form
  \begin{equation}
    \label{eq:20}
    \begin{split}
    \sum_{i=1}^{n} \hat{u}^{(s)}_j \rho_{n} \kappa(X,X_i) &= \rho_{n} \sum_{i=1}^{n}
    \hat{u}^{(s)}_i \sum_{r=1}^{\infty}
    \lambda_{r} \psi_{r}(X) \psi_{r}(X_i) \\
    &= \rho_{n} \sum_{r=1}^{\infty} \sqrt{\lambda_{r}} \psi_{r}(X)
    \sum_{i=1}^{n} \hat{u}^{(s)}_i \sqrt{\lambda}_{r} \psi_{r}(X_i) \\
    &= \rho_{n} \sum_{r=1}^{\infty} \sqrt{\lambda}_{r} \psi_{r}(X) \ast
    \langle \hat{u}^{(s)}, \Psi_{r,n}\rangle_{\mathbb{R}^{n}}
    \end{split}
  \end{equation}
  where $\hat{u}^{(s)}_{j}$ is the $j$-th component of
  $\hat{u}^{(s)}$. Now we note that
\begin{equation*}
  \begin{split}
  \langle \hat{v}^{(s)}, \sqrt{\lambda}_{r} \psi_{r}
  \rangle_{\mathscr{H}} &= \Bigl\langle \frac{1}{\sqrt{\hat{\lambda}_{s}
      n}} \sum_{i=1}^{n} \kappa(\cdot, X_i) \hat{u}^{(s)}_i,
  \sqrt{\lambda}_{r} \psi_{r} \Bigr \rangle_{\mathscr{H}} \\ &=
  \frac{1}{\sqrt{\hat{\lambda}_{s} n}} \sum_{i=1}^{n}
  \hat{u}^{(s)}_{i} \Bigl \langle \kappa(\cdot, X_i), \sqrt{\lambda}_{r}
  \psi_{r} \Bigr \rangle_{\mathscr{H}} \\ 
  &= \frac{1}{\sqrt{\hat{\lambda}_{s}n}}  \sum_{i=1}^{n} \psi_{r}(X_i)
  \sqrt{\lambda}_{r} \hat{u}^{(s)}_i 
  \\ &= \frac{1}{\sqrt{\hat{\lambda}_{s}n}} \langle \hat{u}^{(s)}, \Psi_{r,n} \rangle_{\mathbb{R}^{n}}
  \end{split}
\end{equation*}
where we have used the reproducing kernel property of $\kappa$, i.e.,
$\langle \kappa(\cdot, X), g \rangle_{\mathcal{H}} = g(X)$ for any $g
\in \mathcal{H}$. Thus, the $s$-th component of
$\mathbf{U}_{\mathbf{K}}^{T} \mathbb{E}[\bm{\xi}]$ can be written as
\begin{equation}
  \label{eq:21}
  \rho_{n} \sqrt{\hat{\lambda}_{s}n} \sum_{r=1}^{\infty} \sqrt{\lambda}_{r}
  \psi_{r}(X) \ast \langle \hat{v}^{(s)}, \sqrt{\lambda}_{r} \psi_{r} \rangle_{\mathcal{H}}
\end{equation}
Therefore, as $\tilde{\mathbf{Z}}^{\dagger} =
\mathbf{S}_{\mathbf{K}}^{-1/2} \mathbf{U}_{\mathbf{K}}^{T}$, the
$s$-th component of $\rho_{n}^{-1/2} \tilde{\mathbf{Z}}^{\dagger} \mathbb{E}[\bm{\xi}]$ is just
\begin{equation}
  \label{eq:22}
  \rho_{n}^{-1/2} (\hat{\lambda}_{s} n \rho_n)^{-1/2}
  \mathbf{U}_{\mathbf{K}}^{T} \mathbb{E}[\bm{\xi}] =  \sum_{r=1}^{\infty} \sqrt{\lambda}_{r}
  \psi_{r}(X) \ast \langle \hat{v}^{(s)}, \sqrt{\lambda}_{r} \psi_{r}
  \rangle_{\mathcal{H}} = \hat{v}^{(s)}(X)
\end{equation}
We now consider the projection $\hat{\mathcal{P}}_{d} \kappa(\cdot,
X)$. We have
\begin{equation}
  \label{eq:23}
  \hat{\mathcal{P}}_{d} \kappa(\cdot, X) = \sum_{s=1}^{d} \langle
  \hat{v}^{(s)}, \kappa(\cdot, X) \rangle_{\mathcal{H}} \hat{v}^{(s)}
  = \sum_{s=1}^{d} \hat{v}^{(s)}(X) \hat{v}^{(s)}
\end{equation}
Let us now define $\tilde{T}_n \colon \mathcal{X} \mapsto
\mathbb{R}^{d}$ to be the mapping
 $\tilde{T}_n(X) = \rho_{n}^{-1/2} \tilde{\mathbf{Z}}^{\dagger} \bm{p}_{X}$
where $\bm{p}_{X} = (\rho_{n} \kappa(X_i,X))_{i=1}^{n} \in
\mathbb{R}^{n}$. $\tilde{T}_n$ is a deterministic mapping given the
$\{X_i\}_{i=1}^{n}$ and furthermore, that
\begin{equation}
  \label{eq:25}
  \langle \tilde{T}_n(X), \tilde{T}_n(X') \rangle_{\mathbb{R}^{d}} = \langle
  \hat{\mathcal{P}}_d \kappa(\cdot, X), \hat{\mathcal{P}}_{d}
  \kappa(\cdot, X') \rangle_{\mathcal{H}}
\end{equation}
as the $\{\hat{v}^{(s)}\}$ are orthogonal with respect to $\langle
\cdot, \cdot \rangle_{\mathcal{H}}$. As $\hat{\mathcal{P}_{d}}$ is a
projection onto a finite-dimensional subspace of $\mathcal{H}$, we
thus have that there exists an isometric isomorphism $\imath$ of
the finite-dimensional subspace
of $\mathcal{H}$ spanned by the $\{\hat{v}^{(s)}\}$ into $\mathbb{R}^{d}$
such that $\tilde{T}_n(X) = \rho_{n}^{-1/2} \tilde{\mathbf{Z}}^{\dagger} \bm{p}(X) =
\imath(\hat{\mathcal{P}}_{d} \kappa(\cdot, X))$ for all $X \in
\mathcal{X}$ as desired. 
\end{proof}
Lemma~\ref{lem:4} states that there is an isometric isomorphism
$\hat{\imath}$ of $\mathbb{R}^{d}$ into $\hat{\mathcal{P}}_{d}
\mathcal{H}$ such that $\rho_{n}^{-1/2} \tilde{\mathbf{Z}}^{\dagger} \mathbb{E}[\bm{\xi}] \in
\mathbb{R}^{d}$ is mapped into $\hat{\mathcal{P}}_{d}
\kappa(\cdot, X) \in \hat{\mathcal{P}}_{d} \mathcal{H}$. By the
definition of the truncated feature map $\Phi_{d}$, we also have that
there is an isometric isomorphism $\imath$ of $\mathbb{R}^{d}$ into
$\mathcal{P}_{d} \mathcal{H}$ such that
$\Phi_{d}$ is mapped into $\mathcal{P}_{d}
\kappa(\cdot, X)$. We can thus compare $\rho_{n}^{-1/2} \tilde{\mathbf{Z}}^{\dagger} \mathbb{E}[\bm{\xi}]$ and
$\Phi_{d}$ via their difference in $\mathcal{H}$, i.e., via 
$\| \hat{\imath}(\rho_{n}^{-1/2} \tilde{\mathbf{Z}}^{\dagger} \mathbb{E}[\bm{\xi}]) - \imath(
\Phi_{d}(X))\|_{\mathcal{H}}$. However, a comparison
between $\rho_{n}^{-1/2}\tilde{\mathbf{Z}}^{\dagger} \mathbb{E}[\bm{\xi}]$ and $\Phi_{d}(X)$ as
points in $\mathbb{R}^{d}$ via the Euclidean distance on $\mathbb{R}^{d}$ might be more useful. The
following result facilitates such a comparison. 
\begin{lemma}
  \label{lem:3}
  With probability $1 - 2\eta$ there exists an orthogonal $\mathbf{W}
  \in \mathcal{M}_{d}$ such that
  \begin{equation}
    \label{eq:26}
    \| \rho_{n}^{-1/2} \mathbf{W} \tilde{\mathbf{Z}}^{\dagger} \bm{p}_X - \Phi_{d}(X)
    \|_{\mathbb{R}^{d}} \leq C \sqrt{\frac{\log{(n/\eta)}}{n}}
  \end{equation}
  for all $X \in \mathcal{X}$, where $\bm{p}_X =
  (\rho_{n}\kappa(X_i,X))_{i=1}^{n} = \mathbb{E}[\bm{\xi}] \in \mathbb{R}^{n}$. 
\end{lemma}
\begin{proof}
  Let $\nu,\nu' \in \hat{\mathcal{P}}_{d} \mathcal{H}$ be arbitrary. Thus,
  $\nu = \hat{\mathcal{P}}_{d} \zeta$ and $\nu' =
  \hat{\mathcal{P}}_{d} \zeta'$ for some for some $\zeta,\zeta' \in
  \mathcal{H}$. The polarization identity gives
  \begin{align}
    \label{eq:27}
    \langle \mathcal{P}_{d} \nu, \mathcal{P}_{d} \nu'
    \rangle_{\mathcal{H}} &= \frac{1}{4}\Bigl(\|\mathcal{P}_{d}(\nu +
    \nu')\|_{\mathcal{H}}^{2} - \|\mathcal{P}_{d}(\nu -
    \nu')\|_{\mathcal{H}}^{2}\Bigr)  \\
    \label{eq:30}
    \langle \nu, \nu'
    \rangle_{\mathcal{H}} &= \frac{1}{4}\Bigl(\|\nu +
    \nu'\|_{\mathcal{H}}^{2} - \|\nu - \nu'\|_{\mathcal{H}}^{2}\Bigr) 
  \end{align}
By the Pythagorean theorem, we have 
  \begin{equation}
 \| \zeta - \mathcal{P}_{d} \zeta\|_{\mathcal{H}}^{2} = 
\|\zeta\|_{\mathcal{H}}^{2} - \|\mathcal{P}_{d} \zeta \|_{\mathcal{H}}^{2} 
  \end{equation}
  for any $\zeta \in \mathcal{H}$. Therefore,
  \begin{equation}
    \label{eq:31}
    \begin{split}
    \langle \nu, \nu' \rangle_{\mathcal{H}} - \langle \mathcal{P}_{d}
    \nu, \mathcal{P}_{d} \nu' \rangle_{\mathcal{H}} &= \frac{1}{4} \|
    (\nu + \nu') - \mathcal{P}_{d}(\nu + \nu')\|_{\mathcal{H}}^{2} \\
    &- \frac{1}{4}
    \|(\nu - \nu') - \mathcal{P}_{d}(\nu - \nu')\|_{\mathcal{H}}^{2}
    \end{split}
  \end{equation}
  Eq.~\eqref{eq:31} then implies
  \begin{equation*}
    \label{eq:33}
    \begin{split}
    |\langle \nu, \nu' \rangle_{\mathcal{H}} - \langle \mathcal{P}_{d}
    \nu, \mathcal{P}_{d} \nu' \rangle_{\mathcal{H}}| &\leq \frac{1}{4}
    \sup_{\hat{\mathcal{P}}_{d}
      \mathcal{H}}\{\| \tilde{\nu} -
    \mathcal{P}_{d}\tilde{\nu}\|_{\mathcal{H}}^{2}\} \\ 
    & \leq \frac{1}{4} \sup_{\mathcal{H}} \| \hat{\mathcal{P}}_{d}
    \tilde{\zeta} - \mathcal{P}_{d} \hat{\mathcal{P}}_{d}
    \tilde{\zeta} \|_{\mathcal{H}}^{2} \\
    & \leq \frac{1}{4} \sup_{\mathcal{H}} \| (\hat{\mathcal{P}}_{d} - \mathcal{P}_{d})
    \hat{\mathcal{P}}_{d} \tilde{\zeta} \|_{\mathcal{H}}^{2} \\
    & \leq \frac{1}{4} \| \hat{\mathcal{P}}_{d} - \mathcal{P}_{d}
    \|_{HS(\mathcal{H})}^{2} \|\hat{\mathcal{P}}_{d} \tilde{\zeta}\|_{\mathcal{H}}^{2}
    \end{split}
   \end{equation*}
   Therefore, by Eq.~\eqref{eq:32} in Proposition~\ref{prop:5}, we have for $n$ satisfying
   $4\sqrt{2}n^{-1}\sqrt{\log{(1/\eta)}} < \delta_{d}$, that with probability at least $1 - 2\eta$. 
\begin{equation}
  \label{eq:34}
  |\langle \nu, \nu' \rangle_{\mathcal{H}} - \langle \mathcal{P}_{d}
  \nu, \mathcal{P}_{d} \nu' \rangle_{\mathcal{H}}| \leq
  \frac{2 \log(1/\eta)}{\delta_{d}^{2} n}
\end{equation} holds for all $\nu, \nu' \in \hat{\mathcal{P}}_{d}
\mathcal{H}$. Suppose that Eq.~\eqref{eq:34} holds. Thus, the
projection $\mathcal{P}_{d}$ when restricted to $\hat{\mathcal{P}}_{d}
\mathcal{H}$ is almost an isometry from $\hat{\mathcal{P}}_{d}
\mathcal{H}$ to $\mathcal{P}_{d} \mathcal{H}$. There thus exists a
unique isometry $\tilde{\imath}$ from $\hat{\mathcal{P}}_{d} \mathcal{H}$ to
$\mathcal{P}_{d} \mathcal{H}$ such that
\begin{equation}
  \label{eq:35}
  \| \tilde{\imath}(\nu) - \mathcal{P}_{d} \nu \|_{\mathcal{H}} \leq 3
  \frac{\sqrt{2 \log(1/\eta)}}{\delta_{d} \sqrt{n}}
\end{equation}
holds for all $\nu \in \hat{\mathcal{P}}_{d} \mathcal{H}$ (Theorem 1
in \cite{chmielinski02:_almos}). Now let $\xi \in
\mathcal{P}_{d} \mathcal{H}$ be arbitrary. Then $\xi = \mathcal{P}_{d}
\zeta$ for some $\zeta \in \mathcal{H}$. Let $\nu =
\hat{\mathcal{P}}_{d} \zeta$. We then have
\begin{equation}
  \label{eq:37}
  \begin{split}
  \| \tilde{\imath}(\nu) - \xi \|_{\mathcal{H}} &\leq \| \tilde{\imath}(\nu) -
  \mathcal{P}_{d} \nu \|_{\mathcal{H}} + \| \mathcal{P}_{d} \nu - \xi
  \|_{\mathcal{H}} \\ & \leq
    \| \tilde{\imath}(\nu) -
  \mathcal{P}_{d} \nu \|_{\mathcal{H}} + \| \mathcal{P}_{d}
  \hat{\mathcal{P}}_{d} \zeta - \mathcal{P}_{d} \zeta \| 
  _{\mathcal{H}}  \\
  &\leq 3 \frac{\sqrt{2 \log(1/\eta)}}{\delta_{d} \sqrt{n}} +
  \|\hat{\mathcal{P}}_{d} - \mathcal{P}_{d} \|_{HS(\mathcal{H})} \|\zeta\|_{\mathcal{H}}
  \end{split}
\end{equation}
By Proposition~\ref{prop:5}, the right hand side of
Eq.~\eqref{eq:37} can be bounded to give 
\begin{equation}
  \label{eq:38}
 \|\tilde{\imath}(\nu) - \xi \| \leq C
 \frac{\sqrt{\log{(1/\eta)}}}{\delta_{d} \sqrt{n}}
\end{equation}
for some constant $C$. Thus, for any $\xi \in \mathcal{P}_{d}
\mathcal{H}$, there exists a $\nu \in \hat{\mathcal{P}}_{d}
\mathcal{H}$ with $\| \tilde{\imath}(\nu) - \xi \|_{\mathcal{H}} < C
\delta_{d}^{-1} n^{-1/2} \sqrt{\log(1/\eta)}$, i.e., $\tilde{\imath}$ is $ C
\delta_{d}^{-1} n^{-1/2} \sqrt{\log(1/\eta)}$-surjective. Thus $\tilde{\imath}$ is an
isometric isomorphism from $\hat{\mathcal{P}}_{d} \mathcal{H}$ into
$\mathcal{P}_{d} \mathcal{H}$ (Proposition 1 in
\cite{chmielinski97:_inner}). 

To complete the proof, we note that $(\tilde{\imath} \circ
\hat{\imath}^{-1}) (\tilde{\mathbf{Z}}^{\dagger} \bm{p}_{X})$ is the
image of the isometric isomorphism taking 
$\tilde{\mathbf{Z}}^{\dagger} \bm{p}_{X}$ in  $\mathbb{R}^{d}$ to some
$\nu \in \mathcal{P}_{d} \mathcal{H}$. By our previous reasoning, we
have that $(\tilde{\imath} \circ
\hat{\imath}^{-1}) (\tilde{\mathbf{Z}}^{\dagger} \bm{p}_{X})$ is
``close'' to $\mathcal{P}_{d} \hat{\mathcal{P}}_{d} \kappa(\cdot,
X)$ which is ``close'' to $\mathcal{P}_{d} \mathcal{P}_{d}
\kappa(\cdot, X) = \mathcal{P}_{d} \kappa(\cdot, X)$. Formally, let
$\varphi = \imath \circ \tilde{\imath} \circ \hat{\imath}^{-1}$. We then
have
\begin{equation}
  \label{eq:39}
  \begin{split}
  \| \varphi (\tilde{\mathbf{Z}}^{\dagger} \bm{p}_{X}) - \Phi_{d}(X)
  \|_{\mathbb{R}^{d}} &= \| \tilde{\imath} \circ \hat{\imath}^{-1}
  (\tilde{\mathbf{Z}}^{\dagger} \bm{p}_{X}) - \imath^{-1}(\Phi_{d}(X))
  \|_{\mathcal{H}} \\ 
  &= \| \tilde{\imath}(\hat{\mathcal{P}}_{d} \kappa(\cdot, X)) - \mathcal{P}_{d}
  \kappa(\cdot, X)\|_{\mathcal{H}} \\
  &\leq \| \tilde{\imath}(\hat{\mathcal{P}}_{d} \kappa(\cdot, X)) -
  \mathcal{P}_{d} \hat{\mathcal{P}}_{d} \kappa(\cdot,X)
  \|_{\mathcal{H}} +
  \|\mathcal{P}_{d} \hat{\mathcal{P}}_{d} \kappa(\cdot,X)
   - \mathcal{P}_{d} \kappa(\cdot, X)
  \|_{\mathcal{H}} \\
  & \leq C \frac{\sqrt{\log{(1/\eta)}}}{\delta_{d} \sqrt{n}}.
  \end{split}
\end{equation}
As $\varphi$ is a composition of isometric isomorphisms between finite-dimensional Hilbert spaces, it is an isometric isomorphism from
$\mathbb{R}^{d}$ to $\mathbb{R}^{d}$, i.e., $\varphi$ correspond to an
orthogonal matrix $\mathbf{W}$, as desired.
\end{proof}
 We now proceed to complete the proof of Theorem~\ref{thm:2}.
\begin{proof}[Theorem~\ref{thm:2}]
 Let $\mathbf{W}, \mathbf{W}_{1}, \mathbf{W}_{2} \in
  \mathcal{M}_{d}$ be orthogonal matrices with $\mathbf{W}_{2}
  \mathbf{W}_{1} = \mathbf{W}$. Recall that $\mathbb{E}[\bm{\xi}] =
  (\rho_{n} \kappa(X, X_i))_{i=1}^{n}$. 
  We then have
  \begin{equation*}
    \begin{split}
    \| \rho_{n}^{-1/2} \mathbf{W} T_n(X) - \Phi_{d}(X) \| & 
     \leq \| \rho_{n}^{-1/2}(\mathbf{W}_{2} \mathbf{W}_{1} \mathbf{Z}^{\dagger}
    - \mathbf{W}_{2} \tilde{\mathbf{Z}}^{\dagger}) \bm{\xi} \| + \|
    \rho_{n}^{-1/2}\mathbf{W}_{2} \tilde{\mathbf{Z}}^{\dagger}
    (\bm{\xi} - \mathbb{E}[\bm{\xi}]) \| \\ & + \| \rho_{n}^{-1/2} \mathbf{W}_{2} \tilde{\mathbf{Z}}^{\dagger} \mathbb{E}[\bm{\xi}] -
    \Phi_{d}(X) \| 
  \end{split}
  \end{equation*}
  The first term in the right hand side of the above can be bounded by
  Lemma~\ref{lem:1}, i.e., there exists an orthogonal $\mathbf{W}_{1}$
  such that
    \begin{equation}
      \label{eq:9}
      \| \rho_{n}^{-1/2}(\mathbf{W}_{2} \mathbf{W}_{1} \mathbf{Z}^{\dagger}
    - \mathbf{W}_{2} \tilde{\mathbf{Z}}^{\dagger}) \bm{\xi} \| 
    \leq C \delta_{d}^{-3} \sqrt{\frac{
        \log{(n/\eta)}}{n \rho_n}}
    \end{equation}
    with probability at least $1 - 2\eta$. The second term
    $\|\rho_{n}^{-1/2} \mathbf{W}_{2}
    \tilde{\mathbf{Z}}^{\dagger}(\bm{\xi} - \mathbb{E}[\bm{\xi}])\|$
    can be bounded by Theorem~\ref{thm:4}, i.e.,
   \begin{equation}
      \label{eq:12}
     \|\rho_{n}^{-1/2} \mathbf{W}_{2}
    \tilde{\mathbf{Z}}^{\dagger}(\bm{\xi} - \mathbb{E}[\bm{\xi}])\|
 < C \delta_{d}^{-3} \sqrt{\frac{
        \log{(n/\eta)}}{n \rho_n}}
    \end{equation}
    Finally, the third term is
    bounded by Lemma~\ref{lem:3}. Thus, with probability at least $1 -
    \eta$, there exists some unitary $\mathbf{W}$ such that
   \begin{equation}
      \label{eq:19}
      \| \rho_{n}^{-1/2} \mathbf{W} T_n(X) -
    \Phi_{d}(X) \| \leq C \delta_{d}^{-3}
    \sqrt{\frac{\log{(n/\eta)}}{n \rho_n}} 
    \end{equation}
    as desired.
\end{proof}
\bibliography{oos}
\end{document}